\documentclass{article}

\PassOptionsToPackage{numbers, compress, sort}{natbib}

\usepackage[final]{neurips_2025}

\usepackage[utf8]{inputenc} 
\usepackage[T1]{fontenc}    
\usepackage[colorlinks,citecolor=gray,linkcolor=black]{hyperref}       
\usepackage{url}            
\usepackage{booktabs}       
\usepackage{amsfonts}       
\usepackage{nicefrac}       
\usepackage{microtype}      
\usepackage{xcolor}         
\usepackage{placeins}
\usepackage{graphicx}    
\usepackage{subcaption}  
\usepackage{algorithm}
\usepackage{algorithmicx,algcompatible}
\usepackage{multirow,threeparttable}

\usepackage{amsthm}
\usepackage{wrapfig}
\usepackage{tikz}

\newcommand{\circled}[1]{\tikz[baseline=(char.base)]{
    \node[shape=circle, fill=black, inner sep=0.8pt] (char) {\color{white}\scriptsize #1};
}}
\usepackage{tcolorbox}
\usepackage{listings}
\usepackage{bbding}

\newcommand{\modelname}{R$^2$TL-7B}
\newcommand{\xname}{R$^2$TL}
\newcommand{\modelDistill}{\modelname-Distill\xspace}

\newcommand{\modelFinal}{\modelname\xspace}

\newcommand{\RR}{DeepSeek-R1\xspace}
\newcommand{\OO}{OpenAI-o1\xspace}
\newcommand{\QWENS}{Qwen2.5-Coder-7B-Instruct\xspace}
\newcommand{\QWENL}{Qwen2.5-Coder-32B-Instruct\xspace}
\newcommand{\DSIII}{DeepSeek-V3\xspace}

\newcommand{\sizeDistill}{87K\xspace}      
\newcommand{\sizeRL}{3.1K\xspace}          

\newcommand{\eg}{\textit{e}.\textit{g}.,\xspace}

\usepackage{dsfont} 

\newtheorem{definition}{Definition}[section]
\newtheorem{theorem}{Theorem}[section]

\usepackage{colortbl}
\usepackage{tabularray}
\UseTblrLibrary{booktabs}
\SetTblrInner[booktabs]{abovesep=0pt, belowsep=0pt, rowsep=0.5pt}
\SetTblrInner[booktabs]{cells = {cmd=\small}}

\NewTableCommand\seprule{\specialrule{\lightrulewidth,gray8}{2.5pt}{2.5pt}}
\NewTableCommand\uniquerule{\specialrule{\lightrulewidth,gray7,dashed}{2.5pt}{2.5pt}}
\definecolor{lightb}{RGB}{235,245,255}

\usepackage{graphicx}

\usepackage{pifont}%

\usepackage[nointegrals]{wasysym}
\usepackage{enumitem} %
\setlist[itemize]{leftmargin=*}

\usepackage{xspace}
\usepackage{xcolor}
\usepackage{mathtools}

\usepackage[capitalise, noabbrev]{cleveref}
\crefformat{section}{\S#2#1#3}
\Crefformat{section}{\S#2#1#3}

\usepackage{fontawesome}

\definecolor{codegreen}{rgb}{0,0.6,0}
\definecolor{codegray}{rgb}{0.5,0.5,0.5}
\definecolor{codepurple}{rgb}{0.58,0,0.82}
\definecolor{backcolour}{rgb}{0.95,0.95,0.92}
\definecolor{codeblue}{rgb}{0,0,0.7}

\makeatletter
\newcommand\codefontsize{\@setfontsize\codefontsize\@viiipt\@ixpt}
\makeatother

\renewcommand{\modelname}{CodeV-R1-7B}
\renewcommand{\xname}{CodeV-R1}

\title{QiMeng-\xname: Reasoning-Enhanced Verilog Generation}

\author{%
  \textbf{Yaoyu Zhu$^{1}$, Di Huang$^{1}$\Envelope, Hanqi Lyu$^{1,2}$, Xiaoyun Zhang$^{1,3}$, Chongxiao Li$^{1,3}$,} \\
  \textbf{Wenxuan Shi$^{1,3}$, Yutong Wu$^{1,3}$, Jianan Mu$^{1}$, Jinghua Wang$^{3}$, Yang Zhao$^{1,3}$,} \\
  \textbf{Pengwei Jin$^{1}$, Shuyao Cheng$^{1}$, Shengwen Liang$^{1}$, Xishan Zhang$^{1,4}$,}\\
  \textbf{Rui Zhang$^{1}$, Zidong Du$^{1}$, Qi Guo$^{1}$, Xing Hu$^{1}$,~ Yunji Chen$^{1,3}$} \\
  $^1$ State Key Lab of Processors, Institute of Computing Technology, CAS \\
  $^2$ University of Science and Technology of China \\
  $^3$ University of Chinese Academy of Sciences \\
  $^4$ Cambricon Technologies
  \\\\
  {\url{https://iprc-dip.github.io/CodeV-R1}}
}

\begin{document}

\maketitle
\def\thefootnote{\Envelope}\footnotetext{~Corresponding author. Contact: \{zhuyaoyu, huangdi, huxing\}@ict.ac.cn.}
\def\thefootnote{\arabic{footnote}}

\begin{abstract}
  Large language models (LLMs) trained via reinforcement learning with verifiable reward (RLVR) have achieved breakthroughs on tasks with explicit, automatable verification, such as software programming and mathematical problems. Extending RLVR to electronic design automation (EDA), especially automatically generating hardware description languages (HDLs) like Verilog from natural-language (NL) specifications, however, poses three key challenges: the lack of automated and accurate verification environments, the scarcity of high‐quality NL–code pairs, and the prohibitive computation cost of RLVR. To this end, we introduce \xname, an RLVR framework for training Verilog generation LLMs. First, we develop a rule-based testbench generator that performs robust equivalence checking against golden references. Second, we propose a round-trip data synthesis method that pairs open-source Verilog snippets with LLM‐generated NL descriptions, verifies code–NL–code consistency via the generated testbench, and filters out inequivalent examples to yield a high-quality dataset. Third, we employ a two-stage ``distill-then-RL'' training pipeline: distillation for the cold start of reasoning abilities, followed by adaptive DAPO, our novel RLVR algorithm that can reduce training cost by adaptively adjusting sampling rate. The resulting model, \modelFinal, achieves 68.6 \% and 72.9 \% pass@1 on VerilogEval v2 and RTLLM v1.1, respectively, surpassing prior state-of-the-art by 12$\sim$20 \%, while even exceeding the performance of 671B DeepSeek-R1 on RTLLM. We have released our model, training code, and dataset to facilitate research in EDA and LLM communities. 
  \footnote{Please refer to https://iprc-dip.github.io/CodeV-R1/ for relative resources.}
\end{abstract}

\section{Introduction}
\label{sec:introduction}

Large language models (LLMs) have recently demonstrated remarkable progress on reasoning tasks when trained via reinforcement learning with verifiable reward (RLVR). Notable examples include OpenAI-o1~\cite{openaio1} and DeepSeek-R1~\cite{deepseekr1}, which exhibit emergent reasoning capabilities on problems endowed with explicit verification procedures—such as software programming and mathematical problem solving. This success suggests a promising opportunity to apply RLVR for electronic design automation (EDA), specifically to the automatic generation of hardware description languages (HDLs) like Verilog from natural-language (NL) specifications~\cite{wang2025largelanguagemodelverilog}.

However, the three foundational components required for effective RLVR — (i) a reliable verification environment, (ii) high-quality NL-code data, and (iii) an efficient training algorithm — each present significant challenges in training reasoning LLMs for Verilog generation:

\textbf{(1) Automated verification of hardware designs remains difficult.}
RLVR requires a verification environment capable of providing accurate rewards. However, even in the data-rich software coding domain, such environments are rare. For example, most problems in the programming-contest dataset APPS~\cite{hendrycks2021apps} have only one or two sets of unit tests, and they exhibit a false-positive rate of up to 60\% when evaluated with an average of 20 unit tests~\cite{alphacode}. Consequently, the software community has adopted the practice of using LLMs to generate additional unit tests in order to improve verification quality~\cite{xu2025kodcode, li2023taco}. Nevertheless, this approach is both costly and of limited effectiveness for hardware designs, because LLMs lack the hardware-specific knowledge needed to handle the complex state spaces and corner cases of sequential circuits.
For example, if the reset and clock signals are not correctly configured, the intended functionality cannot be properly verified.

\textbf{(2) High-quality NL–code pairs for hardware designs are scarce.}
The proprietary nature of hardware designs severely limits the availability of annotated Verilog examples. Although several LLM-based methods have been proposed to synthesize NL–code pairs~\cite{pei2024betterv, codev, RTLCoder, gao2024autovcoder}, the resulting datasets often suffer from low-quality data (see Appendix~\ref{sec:case_study} for examples), rendering them inadequate for RLVR’s stringent requirements.

\textbf{(3) The computational cost of RLVR is prohibitive.}
Training a 32B LLM on 1K data for 5 epochs using 16 NVIDIA H100 GPUs with supervised fine-tuning (SFT) takes only 0.5 hours \cite{muennighoff2025s1}. In contrast, training a 14B LLM on 24K verifiable coding problems with reinforcement learning can take over 2.5 weeks on 32 NVIDIA H100 GPUs \cite{deepcoder2025}, making it prohibitively expensive to train a Verilog reasoning LLM using RLVR.

To overcome these challenges, we introduce \xname, a comprehensive RLVR framework for Verilog generation. Our contributions are threefold:

\textbf{(1) Automated testbench generation.}
We develop a rule‑based testbench generation framework to verify the equivalence between a given Verilog implementation and its golden reference as accurately as possible. For each golden reference, the framework first performs circuit‑structure analysis to extract information such as input/output (I/O) ports and reset/clock signals. It then enumerates all reset and clock-synchronization scenarios to improve verification accuracy.
Experiments demonstrate that our testbench achieves 96.1 \% fewer false negatives than the LLM-generated counterpart and detects 62.5 \% more injected errors in fuzzing tests for sequential circuits. Detailed experimental results are presented in Section~\ref{sec:testbench}.

\textbf{(2) Round-trip data synthesis for high-quality NL–code pairs.}
Leveraging our testbench generation framework, we propose the round-trip data synthesis approach that can automatically synthesize high-quality NL–code pairs from code snippets.
Specifically, candidate code snippets are first paired with LLM-generated NL descriptions, and then verified by regenerating the code from NL and comparing against the original for equivalence with our testbench. 
Only code that passes the testbench is retained and combined with the NL to form high-quality data for reinforcement learning.
We theoretically prove that, given strong LLMs and an ideal verification environment, this procedure yields NL–code pairs of sufficiently high quality for RLVR with a high probability.

\textbf{(3) Two-stage training with adaptive DAPO for cost-effective RLVR.}
We adopt a two-stage ``distill-then-RL'' training pipeline to cold-start LLMs' reasoning ability through SFT and apply RL to enhance model's reasoning ability. Specifically, we use DeepSeek-R1 as the NL-to-code LLM in our round-trip data synthesis to produce \texttt{(NL, Thought, Code)} triplets, based on which we perform SFT on our base LLM to obtain a distilled LLM with basic reasoning ability. Then, we apply RLVR on the distilled LLM using the equivalence-checked high-quality data to further enhance its Verilog generation capability.
Additionally, recognizing that RLVR’s bottleneck lies in sampling and verification~\cite{deepcoder2025}, we extend dynamic sampling policy optimization (DAPO)~\cite{yu2025dapo} with an adaptive mechanism that dynamically adjusts the number of samples per training step based on past sample discard rates. This approach notably reduces unnecessary sampling and verification overhead, thereby achieving a 1.25x acceleration.

Based on these techniques, we develop \modelFinal, a specialized reasoning LLM for Verilog generation with only around 2,656 A100-GPU-hours. On the VerilogEval v2~\cite{verilogeval_v2} and RTLLM v1.1~/~v2~\cite{rtllm} benchmarks, \modelFinal achieves 68.8\%~/~72.9\%~/~68.0\% pass@1, respectively. Remarkably, it surpasses the 671B DeepSeek-R1 by 8.1\% on RTLLM v1.1 and 3.3\% on RTLLM v2, demonstrating its strong RTL generation capabilities.

\section{Methods}
\label{sec:method}




\begin{figure}[t]
  \centering
\includegraphics[width=1\linewidth]{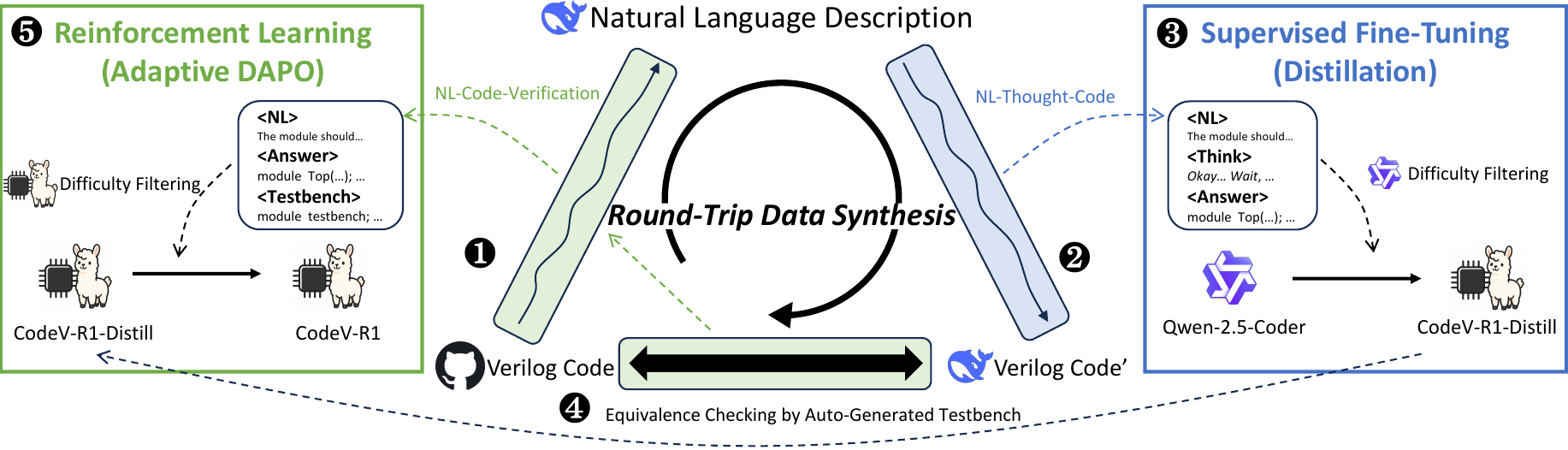}
\caption{\textbf{The overview of \xname. }The core components of our framework include an automated testbench (Section~\ref{sec:equivalence_checking}), a supervised fine-tuning process (Section~\ref{sec:distillation}), and a reinforcement learning process (Section~\ref{sec:reinforcement_learning}). 
}
\vspace{-10pt}
\label{fig:overview}
\end{figure}


Our framework comprises 5 stages, including an automated testbench generation framework (Figure~\ref{fig:overview}). Stages \circled{1} $\sim$ \circled{3} constitute the distillation phase, and stages \circled{4} and \circled{5} comprise the reinforcement learning phase.
Below, we introduce the processes of these phases:

\circled{1} \textbf{Code-to-NL.} Following prior work \cite{codev, mgverilog}, we collect Verilog code snippets from GitHub (denoted $y^*$) and use an LLM (DeepSeek-V3 \cite{deepseekv3}) to produce corresponding natural-language summaries (denoted $x$), creating an NL–code corpus $\{(x_i, y^*_i)\}$ with approximately 150K data samples.

\circled{2} \textbf{NL-to-Code.} Using DeepSeek-R1, we take each NL description $x_i$ from stage 1 and generate the ``thought'' (denoted $c'_i$) as well as an Verilog code snippet (denoted $y'_i$), producing NL–thought–code triples $\{(x_i, c'_i, y'_i)\}$.

\circled{3} \textbf{Supervised Fine-Tuning.} We first filter the $\{(x_i, c'_i, y'_i)\}$ dataset by removing any examples for which base LLMs (\eg \QWENS~/~\QWENL \cite{qwen2.5coder}) can generate correct code in any of 5 attempts (correctness is verified using our automatically generated testbench). We then perform SFT using these data on the base LLM to bootstrap their reasoning ability, yielding the distilled model, \modelFinal-Distill. This stage uses approximately 87K examples.

\circled{4} \textbf{Equivalence Checking.} We use our automated testbench to verify equivalence between the original snippets $y^*$ and the newly generated snippets $y'$. Any non-equivalent pairs $\{(x_i, y^*_i)\}$ are discarded, while equivalent pairs are retained as high-quality data for subsequent RL training.

\circled{5} \textbf{Reinforcement Learning.} We again filter the retained $\{(x_i, y^*_i)\}$ set by removing any examples where the distilled model \modelFinal-Distill generates correct code in all of 5 attempts (as checked by the testbench). After this filtering, approximately \sizeRL examples remain. We then apply our adaptive DAPO algorithm, a novel RLVR algorithm, to further improve Verilog-generation performance, resulting in the final model, \modelFinal.
Next, we will describe in detail the automated testbench generation framework as well as the two training phases, distillation and RL.

\subsection{Automated Testbench Generation Framework for Verilog Code} 
\label{sec:equivalence_checking}

To facilitate the rule-based reward mechanism for the RL process, we have developed a specialized framework. This framework verifies the functionality of the generated Verilog code by conducting edge-triggered simulation and comparing it against the reference code. The verification framework unfolds in three consecutive phases:

\textbf{Phase 1: Circuit-Structure Analysis.}
Before performing functional verification, we extract the input/output (I/O) ports along with their respective bit-widths from the reference golden code using Yosys~\cite{yosys}. For sequential circuits, we identify clock signals, noting their edge polarity (rising or falling), and characterize reset signals through control flow analysis. Reset signals are categorized based on synchrony (synchronous if they depend on the clock) and polarity (active-high or active-low).

\textbf{Phase 2: Simulation.}
We simulate by providing random inputs to both the generated and the reference codes, and evaluating the equivalence of outputs.
For \textbf{combinational circuits}, we employ \(M\) = 100 independent simulation sequences for equivalence evaluation, each comprising \(N\) = 1000 inputs.
Regarding \textbf{sequential circuits}, we adopt a dual-stage validation approach when dealing with circuits that have either one reset signal or no reset signal at all:
Firstly, we execute simulations using \(M\) = 100 sequences, each with \(N\) = 1000 clock toggles (500 cycles) with randomized inputs. In this stage, deterministic reset signals—derived from golden reset behavior extracted via Yosys and representing expected, consistent reset logic—are applied at the start of each sequence, primarily aimed at testing the circuit's core functionality.
Secondly, we conduct simulations with an identical number of sequences and clock cycles with random reset signals, which validates the consistency of the reset signal operation.
For circuit designs featuring multiple reset signals, we exhaustively test every non-conflicting combination, all maintaining the aforementioned \(\frac{MN}{2}\) cycle count.

\textbf{Phase 3: Verification.}
After each clock toggle, we assess the equivalence of the outputs between the generated Verilog code and the reference implementation. This process results in a total of $2MN$ assessments for typical sequential circuits and $MN$ assessments for combinational circuits. The verification outcome is quantified by an error rate metric \(\epsilon = \frac{\text{Error Number}}{2MN} \times 100\%\). A value of \(\epsilon = 0\%\) indicates that the generated code functions correctly within our testbench environment.
Through 32-way parallelization, the simulation achieves a throughput of 15 instances per second.

\subsection{\modelDistill: Supervised Distillation for Verilog Data}
\label{sec:distillation}

Our pipeline for distillation begins with a set of Verilog code (denoted as $y^*_i$) collected from GitHub. We use \DSIII to summarize these code snippets, producing instructions $x_i$ corresponding to $y^*_i$ (stage \circled{1}).
Then, to produce the corpus for distillation, we ask \RR to generate responses containing "thought" $c'_i$ and Verilog code snippet $y'_i$ (stage \circled{2}).
These two stages yield approximately 150K NL-thought-code triples $(x_i, c'_i, y'_i)$.

Next, we curate a \textit{challenging subset} through two filters: (1) retaining only instructions where baseline models (\QWENS~/~\QWENL) fail to generate the code passing the functional verification (Section~\ref{sec:equivalence_checking}) to $y^*_i$, and (2) ensuring synthesizability of $y^*_i$ with Yosys~\cite{yosys}.
In addition, to prevent benchmark contamination, we remove samples where the generated code $y'_i$ exhibits Rouge-L similarity > 0.5~\cite{rouge-l} to VerilogEval v1~\cite{liu2023verilogeval} / v2~\cite{verilogeval_v2} or RTLLM v1.1~\cite{rtllm} / v2~\cite{rtllm_v2}, yielding \sizeDistill high-quality samples (stage \circled{3}).

Finally, we initialize \modelDistill from \QWENS and fine-tune it to generate complete responses ($c'_i$, $y'_i$) given $x_i$. Following \RR's methodology~\cite{deepseekr1}, we maximize the likelihood of the generated responses using our prompt template (see Appendix~\ref{sec:prompts}), with implementation specifics detailed in Section~\ref{sec:implementation_details} (stage \circled{3}).

\subsection{\modelFinal: Reinforcement Learning on the Distilled Model}
\label{sec:reinforcement_learning}
To further improve the model's reasoning ability, we perform reinforcement learning fine-tuning based on \modelDistill with carefully selected high-quality Verilog data (stage \circled{4} and stage \circled{5}).
Below we will introduce our data curation method (Section \ref{sec:data_curation}), RL training algorithm, and reward design for RL (Section \ref{sec:RL_algorithm}). 

\subsubsection{High-quality Data Curation}
\label{sec:data_curation}


Experiences from prior research suggest that conducting RL training on problems that \textit{the model can solve but requires reasoning to address} can more effectively enhance the model's RL capabilities \cite{kimik1.5}. 
Furthermore, given potential inconsistencies between the golden code $\{y^*_i\}$ in the original dataset collected from GitHub and the instructions $\{x_i\}$ generated by \DSIII, we prioritize \textit{ensuring the correctness} of selected problems. 
To summarize, our RL (question, answer) pairs must meet three key criteria: being \textbf{solvable}, \textbf{challenging}, and \textbf{error-free}. 


To implement this framework, we identify problems where \RR~\textit{successfully generates code matching the golden one in the original dataset}, while both \QWENS~and \QWENL~\textit{fail to produce equivalent solutions}. 
Specifically, we conduct equivalence checking between the $\{y'_i\}$ code generated by \RR~in the \sizeDistill~dataset and $\{y^*_i\}$ in the original dataset, retaining only \textit{validated} $\{(x'_i, y^*_i)\}$ pairs for RL training. 

For difficulty enhancement, we employ \modelDistill~to generate five code variants per question, excluding cases where all generated codes match the golden one, as these reflect patterns already mastered during supervised fine-tuning (stage \circled{4}). 
Through this rigorous selection process, we curate a final dataset of \sizeRL~\textit{high-quality} examples for reinforcement learning.

Additionally, we formalize the equivalence between code and natural language, and theoretically prove the effectiveness of our data curation. Intuitively, the Code-to-NL and NL-to-Code conversion process using LLMs inevitably leads to some information loss. Therefore, if the converted code remains equivalent to the original one after back-and-forth conversion, the probability of error during the conversion process is minimal. Detailed definition and proof are shown below.

\begin{definition}[NL-Code Deterministic Equivalence (NLCDE)]\label{def:nlcde}

Let \(\mathcal{F}\) denote the space of all code snippets, \(\mathcal{L}\) the space of natural-language (NL) descriptions, and \(\mathcal{R} \subseteq \mathcal{F} \times \mathcal{L}\) a \emph{semantic/functional equivalence relation} where \((f, l) \in \mathcal{R}\) iff code \(f\) fully implements NL \(l\) (or \(l\) precisely describes \(f\)).  

Consider two probabilistic models, $M_1: \mathcal{F} \to \mathcal{L} \ (\text{code-to-NL})$ and $ M_2: \mathcal{L} \to \mathcal{F} \ (\text{NL-to-code})$, the \textbf{NLCDE} states:  
For all \(f \in \mathcal{F}, l \in \mathcal{L}\):  
1. If \(M_1\) generates \(l\) with \(\Pr(l \mid f) = 1\), then \((f, l) \in \mathcal{R}\) (deterministic NL summaries are semantically equivalent to input code).  
2. If \(M_2\) generates \(f\) with \(\Pr(f \mid l) = 1\), then \((f, l) \in \mathcal{R}\) (deterministic code outputs are functionally equivalent to input NL).  
\end{definition}

\begin{theorem}[Semantic Equivalence in Round-Trip Transformations] \label{thm:semantic_equivalence}
Consider the probabilistic models \(M_1: \mathcal{F} \to \mathcal{L}\) (code-to-NL) and \(M_2: \mathcal{L} \to \mathcal{F}\) (NL-to-code) from the \textbf{NL-Code Deterministic Equivalence (NLCDE)} definition (Definition~\ref{def:nlcde}). Let \(Y \in \mathcal{F}\) be a random code snippet drawn from some distribution, and define the transformed objects:
$
X = M_1(Y) \in \mathcal{L},~ Y' = M_2(X) \in \mathcal{F}.
$
For any pair of objects \(A,B\), let \(E_{AB}\) denote the event “\(A\) and \(B\) are semantically equivalent.” If the round-trip transformation preserves equivalence with certainty under NLCDE, i.e.,
$
\Pr\bigl[E_{Y,Y'}\bigr] = 1,
$
then both forward and backward transformations are individually equivalent with certainty:
$
\Pr\bigl[E_{Y,X} \land E_{X,Y'}\bigr] = 1.
$
\end{theorem}
\begin{proof}[Proof Sketch]
This theorem can be proved by the Data Processing Inequality. Please refer to Appendix~\ref{sec:proof} for the detailed proof and further explanation.
\end{proof}

\subsubsection{Adaptive DAPO Algorithm}
\label{sec:RL_algorithm}

We enhance the DAPO algorithm \cite{yu2025dapo} with two efficiency improvements for RL fine-tuning on the distilled model (stage \circled{5}). The core DAPO loss operates on groups of $G$ responses per prompt:

\begin{align}
    \mathcal{L}_{DAPO}(\theta) = &\mathbb{E}_{(x,y^*) \sim \mathcal{D}, \{y_i\}_{i=1}^G \sim \pi_{\theta_{old}}(\cdot|x)} \nonumber \\
    &\left[ \frac{1}{\sum_{i=1}^G |y_i|} \sum_{i=1}^G \sum_{t=1}^{|y_i|} \min \left( r_{i,t}(\theta) \hat{A}_{i,t}, \text{clip} \left( r_{i,t}(\theta), 1-\epsilon_{low}, 1+\epsilon_{high} \right) \hat{A}_{i,t} \right) \right], \label{eq:dapo} \\
    \text{s.t.} \quad & 0 < |\{ y_i | \text{is\_equivalent}(y_i, y^*) \}| < G, \nonumber
\end{align}
where $r_{i,t}(\theta) = \frac{\pi_\theta(y_{i,t}|x, y_{i, <t})}{\pi_{\theta_{old}}(y_{i,t}|x, y_{i, <t})}$, $\hat{A}_{i,t} = \frac{R_i - \text{mean}(\{R_i\}_{i=1}^G) }{ std(\{R_i\}_{i=1}^G)}$ ($R_i$ is the reward to be introduced later)
, $r_{i,t}$ is the importance sampling ratio under the new policy $\pi_\theta$ compared to the old policy $\pi_{\theta_{old}}$, $\hat{A}_{i,t}$ is the group-relative advantage, $|y_i|$ is the length of response to calculate token-level loss, and $\epsilon_{low} < \epsilon_{high}$ are asymmetric clipping thresholds introduced in DAPO to encourage exploration. 
The constraint $0 < |\{ y_i | \text{is\_equivalent}(y_i, y^*) \}| < G$ ensures each training batch contains both correct and incorrect responses.
Note that we do not include the overlong filtering proposed by DAPO here.


    
    
    
    

A key feature of DAPO is the dynamic sampling mechanism. It notably improves the training result. However, the standard DAPO sampling strategy presents inefficiencies during sample generation: 
DAPO's fixed generation batch size (denoted as $b_{gen}$) is suboptimal. If too few partially correct samples are generated for the RL train batch size (denoted as $b_{train}$), costly re-sampling occurs; if too many are generated, excess samples are wasted. This problem intensifies as training progresses, and the improvement in model accuracy reduces the number of partially correct examples. 

We address this with an \textbf{adaptive batch size} mechanism utilizing a dynamically estimated sampling effective ratio, $r_{valid}$. 
Initially, $b_{gen}$ is set to $b_{train}$. After successfully accumulating a full training batch ($b_{train}$), we calculate the batch effective ratio ($\frac{\text{number of valid samples}}{b_{gen}}$). The value of $r_{valid}$ is then updated to the \textit{minimum} of itself and the batch effective ratio. For the subsequent sampling phase, the generation batch size is adaptively set to $b_{gen} = \lceil \frac{b_{train}}{r_{valid}} \rceil$. 
The detailed process is given in Appendix~\ref{sec:adaptive_dapo}. 
Note that this acceleration does not involve offline updates or alter the composition of the RL training batch, so it preserves DAPO's accuracy while accelerating training.

We implement a rule-based reward function that evaluates both structural correctness and semantic equivalence. A response $y_i$ receives a reward of 1 if it satisfies two conditions:
(1) Proper formatting as ``<think>reasoning</think><answer>solution</answer>''
(2) Semantic equivalence with the golden code $y^*$ judged by the equivalence checker introduced in Section \ref{sec:equivalence_checking}.
The reward function $R(y, y^*)$ is 1 if $y$ has a correct format and $(y, y^*)$ are functional equivalent, and 0 otherwise.





\section{Experiments}
\label{sec:experiments}

This section details the implementation of our method and presents comprehensive experimental results. 
We systematically evaluate our model through multiple dimensions: comparisons with prior state-of-the-art approaches, test-time scaling analysis across varying response length constraints, ablation studies analyzing the impact of golden code correctness and problem complexity, acceleration effects of the adaptive DAPO mechanism, and testbench performance evaluation.
These analyses collectively demonstrate the effectiveness and efficiency of our proposed approach.


\subsection{Implementation details}
\label{sec:implementation_details}

We obtain our final model by first distilling \RR and then applying RL on our curated \sizeRL dataset.
During distillation, we employ LLaMAFactory \cite{zheng2024llamafactory} to supervised fine-tune (SFT) Qwen2.5-Coder-7B-Instruct using the \sizeDistill dataset filtered for distillation. We train the model for 6 epochs with a learning rate of $1 \times 10^{-5}$ and a batch size of 64. The total context length is set to 16384 during distillation.
During RL, we use the verl \cite{sheng2024hybridflow} framework to further train the distilled model with our adaptive DAPO.
We use a batch size of 128, a learning rate of $1 \times 10^{-6}$, and train for 300 steps. The rollout temperature is set to $1.0$.
During this stage, the max length is set to 2048 for instruction and 16384 for response.
The SFT stage is executed on 8 A100-80G GPUs, taking approximately 78 hours, while the RL stage runs on 16 A100-80G GPUs, requiring around 127 hours of computation.
The whole parameter setting is provided in Appendix~\ref{sec:parameter}.

We test our distillation and RL model on various Verilog benchmarks, including VerilogEval v1 \cite{liu2023verilogeval} / v2 \cite{verilogeval_v2} and RTLLM v1.1 \cite{rtllm} / v2 \cite{rtllm_v2}. For VerilogEval v2, we examine zero-shot scenarios in both specification-to-RTL translation and code completion tasks. 
The maximum context length is configured to 16384 tokens during the evaluation phase for all benchmarks. The temperature during generation is 0.6 for the distillation model and 1.0 for the RL model, and 20 responses are generated per query to estimate the pass@k score for both VerilogEval and RTLLM.


\subsection{Main Results}
\label{sec:main_results}

\begin{table*}[h]
\caption{Comparison of \modelFinal against baselines on VerilogEval v1 and RTLLM v1.1.}
\label{tab:main_exp}
\centering
\footnotesize
\scalebox{0.71}{ 
{%
    \begin{tabular}{clccccccccc}
    \toprule
    \multirow{2}{*}{Type} & \multirow{2}{*}{Model} & \multirow{2}{*}{\begin{tabular}[c]{@{}c@{}}Open\\ source\end{tabular}} & \multicolumn{3}{c}{VerilogEval-Machine (\%)} & \multicolumn{3}{c}{VerilogEval-Human (\%)} & \multicolumn{2}{c}{RTLLM v1.1 (\%)} \\
    & & &  pass@1 & pass@5 & pass@10 & pass@1 & pass@5 & pass@10 & pass@1 & pass@5 \\ 
    \midrule
    \multirow{8}{*}{\begin{tabular}[c]{@{}c@{}}Foundation\\ Models\end{tabular}}
    & GPT-4o* & \text{\texttimes} & 67.7 & 75.5 & 77.2 & 60.1 & 71.4 & 74.5 & 41.7 & 65.9 \\
    & DeepSeek-R1-671B* & \checkmark & \underline{81.0} & 87.4 & \underline{89.5} & \textbf{81.5} & \textbf{87.6} & \textbf{88.5} & \underline{64.8} & \underline{82.9} \\
    & DeepSeek-V3-671B* & \checkmark & 80.8 & 87.5 & 88.8 & 68.7 & \underline{79.7} & \underline{82.1} & 60.9 & 74.2 \\
    & QWQ-32B* & \checkmark & 71.1 & 84.0 & 87.0 & 63.8 & 78.0 & 81.3 & 50.9 & 70.6 \\
    & DeepSeek-R1-Distill-Qwen-32B* & \checkmark & 64.7 & 80.5 & 83.6 & 51.3 & 68.1 & 72.2 & 42.1 & 64.3 \\
    & DeepSeek-R1-Distill-Qwen-7B* & \checkmark & 5.3 & 16.9 & 24.9 & 1.6 & 6.3 & 10.1 & 0.0 & 0.0 \\
    & Qwen2.5-Coder-32B-Instruct* & \checkmark & 66.6 & 76.6 & 79.7 & 47.6 & 58.1 & 61.8 & 47.9 & 67.7 \\
    & Qwen2.5-Coder-7B-Instruct* & \checkmark & 60.2 & 77.8 & 82.4 & 31.9 & 46.3 & 50.2 & 32.2 & 48.2 \\
     \midrule
    \multirow{11}{*}{\begin{tabular}[c]{@{}c@{}}Specialized\\ Models\end{tabular}}
    & RTLCoder-Mistral-7B & \checkmark & 62.5 & 72.2 & 76.6 & 36.7 & 45.5 & 49.2 & {-} & 48.3 \\
    & RTLCoder-DS-6.7B & \checkmark & 61.2 & 76.5 & {81.8} & 41.6 & 50.1 & 53.4 & {-} & 48.3 \\
    & BetterV-CL-7B & \text{\texttimes} & 64.2 & 75.4 & 79.1 & 40.9 & 50.0 & 53.3 & - & - \\
    & BetterV-DS-6.7B & \text{\texttimes} & 67.8 & 79.1 & 84.0 & 45.9 & 53.3 & 57.6 & - & - \\
    & BetterV-CQ-7B & \text{\texttimes} & 68.1 & 79.4 & {84.5} & 46.1 & 53.7 & 58.2 & - & - \\
    & CodeV-CL-7B & \checkmark & 78.1 & 86.0 & 88.5 & 45.2 & 59.5 & 63.8 & 39.4 & 62.1 \\
    & CodeV-DS-6.7B & \checkmark & 77.9 & \textbf{88.6} & \textbf{90.7} & 52.7 & 62.5 & 67.3 & 42.4 & 55.2 \\
    & CodeV-CQ-7B & \checkmark & 77.6 & \underline{88.2} & \textbf{90.7} & 53.2 & 65.1 & 68.5 & 36.6 & 55.2 \\
    & CraftRTL-CL-7B & \text{\texttimes} & 78.1 & 85.5 & 87.8 & 63.1 & 67.8 & 69.7 & 42.6 & 52.9 \\
    & CraftRTL-DS-6.7B & \text{\texttimes} & 77.8 & 85.5 & 88.1 & 65.4 & 70.0 & 72.1 & 53.1 & 58.8 \\
    & CraftRTL-SC2-15B & \text{\texttimes} & \textbf{81.9} & 86.9 & 88.1 & 68.0 & 72.4 & 74.6 & 49.0 & 65.8 \\
     \midrule
    \rowcolor[rgb]{0.925,0.925,0.925}
    & \modelDistill & \checkmark & 76.2 & 85.6 & 87.0 & 65.7 & 76.8 & 79.7 & 57.4 & 75.8 \\ 
    \rowcolor[rgb]{0.925,0.925,0.925}
    \multirow{-2}{*}{Ours} & \modelFinal & \checkmark & 76.5 & 84.1 & 85.7 & \underline{69.9} & 79.3 & 81.7 & \textbf{72.9} & \textbf{86.1} \\ 
     \bottomrule
    \end{tabular}%
}
}

\begin{tablenotes}
\item $^{\rm *}$ We evaluate the models with *, while other results are sourced from their papers.
\end{tablenotes}
\vspace{-10pt}
\end{table*}

\begin{table*}[h]
\caption{Comparison of \modelFinal on VerilogEval v2 and RTLLM v2.}
\label{tab:main_exp_2}
\centering
\footnotesize
\scalebox{0.63}{ 
\begin{tabular}{clc ccc ccc ccc} 
\toprule
\multirow{2}{*}{Type} & \multirow{2}{*}{Model} & \multirow{2}{*}{\begin{tabular}[c]{@{}c@{}}Open\\ source\end{tabular}} & \multicolumn{3}{c}{VerilogEval2-SR (\%)} & \multicolumn{3}{c}{VerilogEval2-CC (\%)} & \multicolumn{3}{c}{RTLLM v2 (\%)} \\ 
& & & pass@1 & pass@5 & pass@10 & pass@1 & pass@5 & pass@10 & pass@1 & pass@5 & pass@10 \\ 
\midrule
\multirow{8}{*}{\begin{tabular}[c]{@{}c@{}}Foundation\\ Models\end{tabular}}
& GPT-4o & \text{\texttimes} & 64.1 & 73.7 & 76.2 & 57.6 & 66.1 & 69.0 & 56.5 & 70.3 & 75.2 \\ 
& DeepSeek-R1-671B & \checkmark & \textbf{77.5} & \textbf{84.7} & \textbf{87.4} & \textbf{79.1} & \textbf{85.1} & \textbf{87.1} & \underline{64.7} & \underline{75.8} & \underline{79.7} \\ 
& DeepSeek-V3-671B & \checkmark & 62.4 & 71.7 & 75.0 & 68.7 & 76.3 & 78.2 & 59.1 & 71.5 & 73.3 \\ 
& QWQ-32B & \checkmark & 64.2 & 77.3 & 80.1 & 64.0 & 77.8 & \underline{80.9} & 52.9 & 68.0 & 71.2 \\ 
& DeepSeek-R1-Distill-Qwen-32B & \checkmark & 43.9 & 63.3 & 69.2 & 53.8 & 69.8 & 73.8 & 42.4 & 62.1 & 67.0 \\ 
& DeepSeek-R1-Distill-Qwen-7B & \checkmark & 0.6 & 2.2 & 3.5 & 2.0 & 7.0 & 11.3 & 0.0 & 0.0 & 0.0 \\ 
& Qwen2.5-Coder-32B-Instruct & \checkmark & 47.5 & 60.7 & 64.7 & 46.6 & 59.0 & 62.8 & 47.8 & 63.9 & 67.8 \\ 
& Qwen2.5-Coder-7B-Instruct & \checkmark & 31.3 & 49.3 & 54.6 & 30.5 & 46.8 & 52.0 & 36.1 & 52.4 & 57.6 \\ 
\midrule
\multirow{1}{*}{Specialized Models}
& RTLCoder-DS-6.7B & \checkmark & 31.1 & 47.8 & 52.3 & 33.7 & 45.9 & 49.8 & 33.6 & 45.3 & 49.2 \\ 
\midrule
\rowcolor[rgb]{0.925,0.925,0.925}
& \modelDistill & \checkmark & 65.2 & 75.2 & 77.5 & 65.5 & 75.6 & 78.2 & 57.2 & 71.9 & 77.1 \\ 
\rowcolor[rgb]{0.925,0.925,0.925}
\multirow{-2}{*}{Ours} & \modelFinal & \checkmark & \underline{68.8} & \underline{78.2} & \underline{81.1} & \underline{69.9} & \underline{78.2} & \underline{80.9} & \textbf{68.0} & \textbf{78.2} & \textbf{81.7} \\ 
\bottomrule
\end{tabular}
}

\begin{tablenotes}
\item $^{\rm *}$ We evaluate all models in this table. SR: Specification-to-RTL; CC: Code Completion.
\end{tablenotes}
\vspace{-15pt}
\end{table*}

Our main experimental results are shown in Table~\ref{tab:main_exp} and Table~\ref{tab:main_exp_2}.
We evaluate DeepSeek-R1~\cite{deepseekr1}, DeepSeek-V3~\cite{deepseekv3}, QWQ-32B~\cite{qwq32b}, DeepSeek-R1-Distill-Qwen-32B~\cite{deepseekr1}, DeepSeek-R1-Distill-Qwen-7B~\cite{deepseekr1}, Qwen2.5-Coder-32B-Instruct~\cite{qwen2.5}, Qwen2.5-Coder-7B-Instruct~\cite{qwen2.5}, and GPT-4o \cite{openai2024gpt4ocard} on VerilogEval  and RTLLM.
Meanwhile, we adopt results reported by RTLCoder~\cite{RTLCoder}, BetterV~\cite{pei2024betterv}, CodeV~\cite{codev}, CraftRTL~\cite{liu2024craftrtl} from their papers.
The results demonstrate that:

\paragraph{Our model achieves state-of-the-art (SOTA) performance among Verilog-domain models on most benchmarks.} 
Our model has a significant advantage over previous Verilog-domain models on RTLLM v1.1, outperforming the previous SOTA model, CraftRTL-DS-6.7B, by 18.8\% on the pass@1 metric.
On VerilogEval v1-Human, although the performance improvement compared to the previous SOTA model, CraftRTL-SC2-15B, is not substantial, our model has a smaller size (7B) compared to theirs (15B). Among 7B models, we outperform the previous best model (CraftRTL-DS-6.7B) by 4.5\% on pass@1.
Although our model does not perform well on VerilogEval-Machine, this benchmark is relatively easy, and even \RR does not have a significant advantage on it.

\paragraph{Our model demonstrates superior performance over most foundation models across both benchmarks.} Although it does not surpass the \RR model—the primary source for knowledge distillation—on most benchmarks, it consistently exceeds the performance of other foundation models. A key finding is that after applying reinforcement learning (RL), \textbf{our model outperforms \RR on both RTLLM-v1.1 and RTLLM-v2}, underscoring the significant efficacy of the RL phase. The underwhelming results of other foundation models, such as Qwen2.5-Coder-Instruct and DeepSeek-R1-Distill-Qwen, highlight the limited exposure to Verilog data during their pre-training and instruction-tuning stages. This is further evidenced by the observation that distilling general-purpose knowledge from large models (e.g., in mathematics and software code) fails to enhance the Verilog capabilities of smaller models.

\paragraph{Reinforcement learning significantly improves model performance.} Compared with \modelDistill, our RL model \modelFinal shows a noticeable improvement on almost all benchmarks.
Especially on the RTLLM benchmark, the reinforcement learning process results in an improvement of over 10 \% for the pass@1 score.
This indicates great potential of RL for Verilog code generation and showcases the robustness of our testbench in providing reliable functional correctness rewards.

\subsection{Additional Experiments}

\subsubsection{Test-Time Scaling}

\begin{figure}[h]
    \centering
    \begin{subfigure}{0.45\linewidth}
    \includegraphics[width=\linewidth]{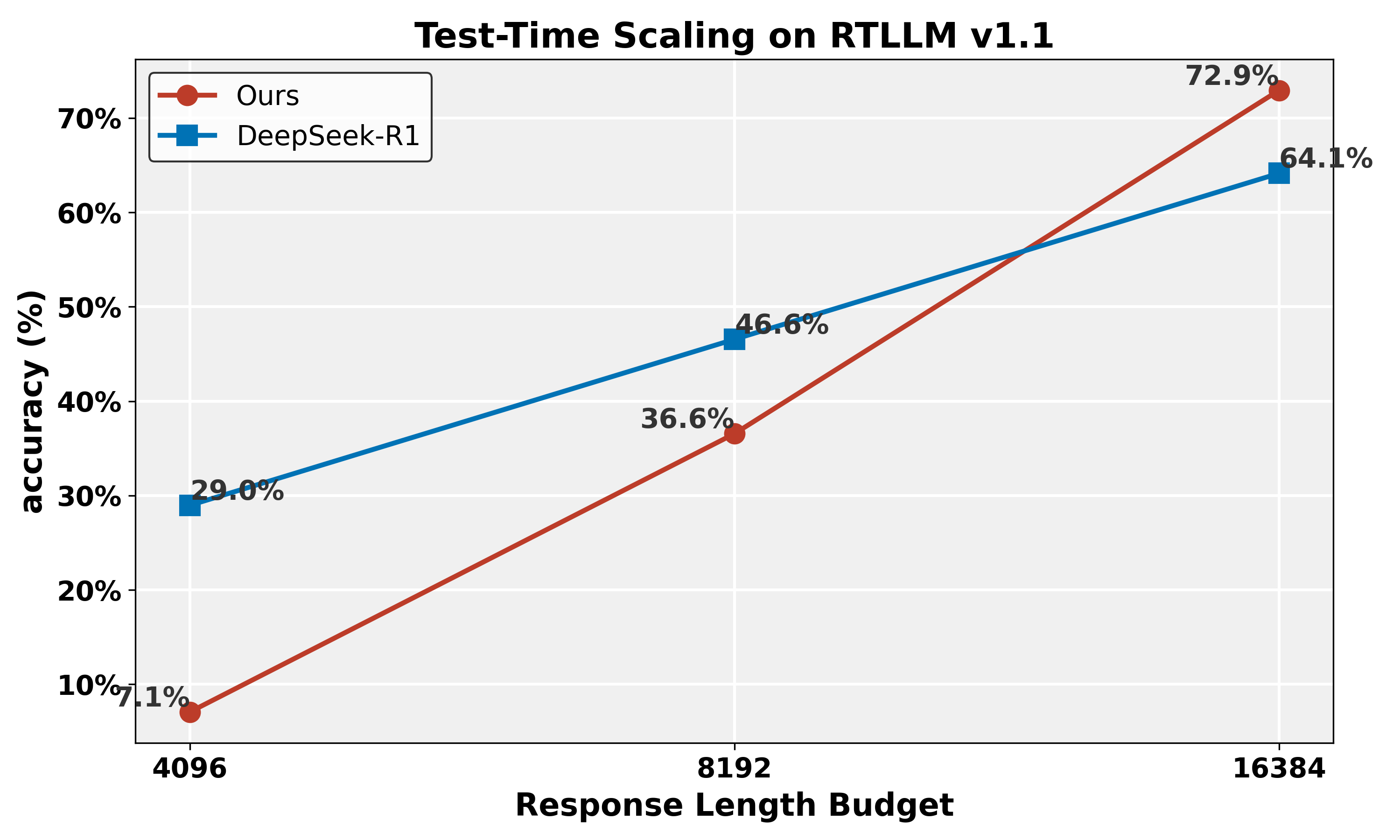}
    \caption{\label{fig:tts_length}}
    \end{subfigure}
    \hspace{0.03\linewidth} 
    \begin{subfigure}{0.45\linewidth}
    \includegraphics[width=\linewidth]{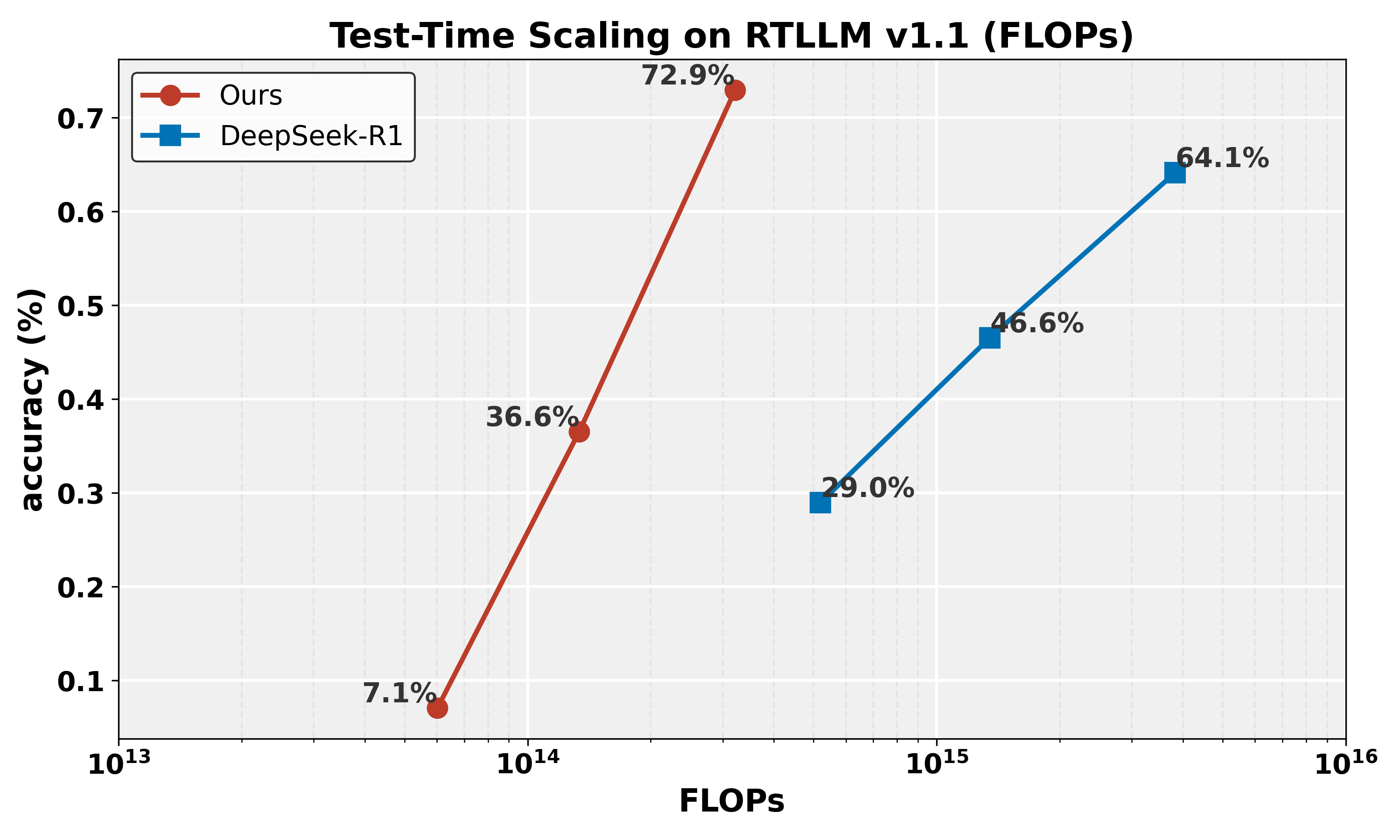}
    \caption{\label{fig:tts_flops}}
    \end{subfigure}
    \caption{\textbf{Test-time scaling on RTLLM v1.1.} Figure~(a) shows response length against accuracy, while Figure~(b) shows FLOPs against accuracy. FLOPs are estimated according to model architecture.}
    \label{fig:tts}
\end{figure}

Test-time scaling is an important ability of reasoning LLMs~\cite{muennighoff2025s1}. 
To verify the test-time scaling ability of our \modelFinal, we take the RTLLM v1.1 dataset as an example and evaluate the accuracy of our model and \RR under varying response length budgets.
Formally, we force the response length of both models to be smaller than certain thresholds (4096, 8192 and 16384 tokens), and plot the corresponding results in Figure~\ref{fig:tts_length}.
To ensure fair comparison, we also normalized FLOPs consumption at each response length, as shown in Figure~\ref{fig:tts_flops}.

Both models' accuracy improves considerably as the response length budget increases from 4096 to 16384. \modelFinal's accuracy rises from 7.1\% to 72.9\%, outperforming \RR (29.0\% $\to$ 64.1\%). 
When evaluated in terms of FLOPs efficiency, \modelFinal demonstrated superior computational economy, delivering higher accuracy per unit of computation compared to \RR.
These results underscore \modelFinal's exceptional test-time scaling efficiency, showcasing its ability to leverage longer contexts more effectively than \RR while consuming fewer computational resources on the RTLLM v1.1 benchmark.

\subsubsection{Equivalence Checking and Difficulty Filtering Improves RL Training}

To explore whether equivalence checking and difficulty filtering improve RL dataset quality, we conduct an ablation study by constructing two additional datasets.

Our original RL dataset contains \sizeRL problems where DeepSeek-R1 responses pass the equivalence checking, while both Qwen2.5-Coder-7B-Instruct and Qwen2.5-Coder-32B-Instruct fail across five sampling attempts. 
To conduct difficulty ablation, we introduce a \textbf{dataset without difficulty filtering} containing 16K problems, where we additionally include samples where Qwen2.5 models succeed in some attempts under our testbench.
To conduct reference code correctness ablation, we introduce a \textbf{dataset without round-trip equivalence checking} containing 14K samples, where we treat DeepSeek-R1 outputs as pseudo-golden code. We select cases where Qwen2.5 models fail to match this pseudo-golden code in five attempts to control difficulty.
To avoid time waste, we filter the problems where \modelDistill has a 100\% pass rate under our testbench in five attempts.

\begin{figure}[h]
  \centering
  \begin{subfigure}{0.4\linewidth}
    \centering
    \includegraphics[width=\linewidth]{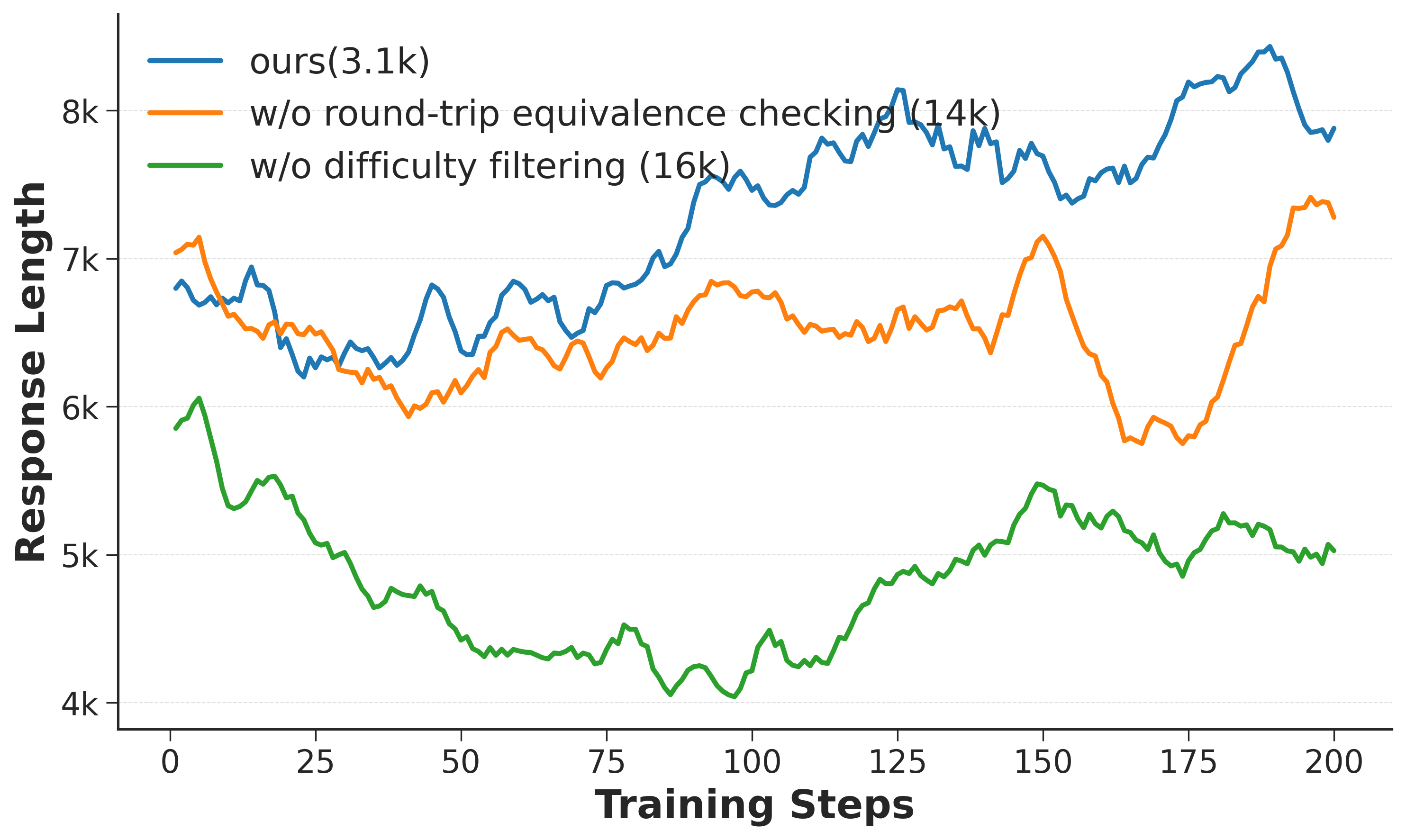}
    \caption{\label{fig:ablation_train_resp_len}}
  \end{subfigure}
  \hspace{0.1\linewidth} 
  \begin{subfigure}{0.4\linewidth}
    \centering
    \includegraphics[width=\linewidth]{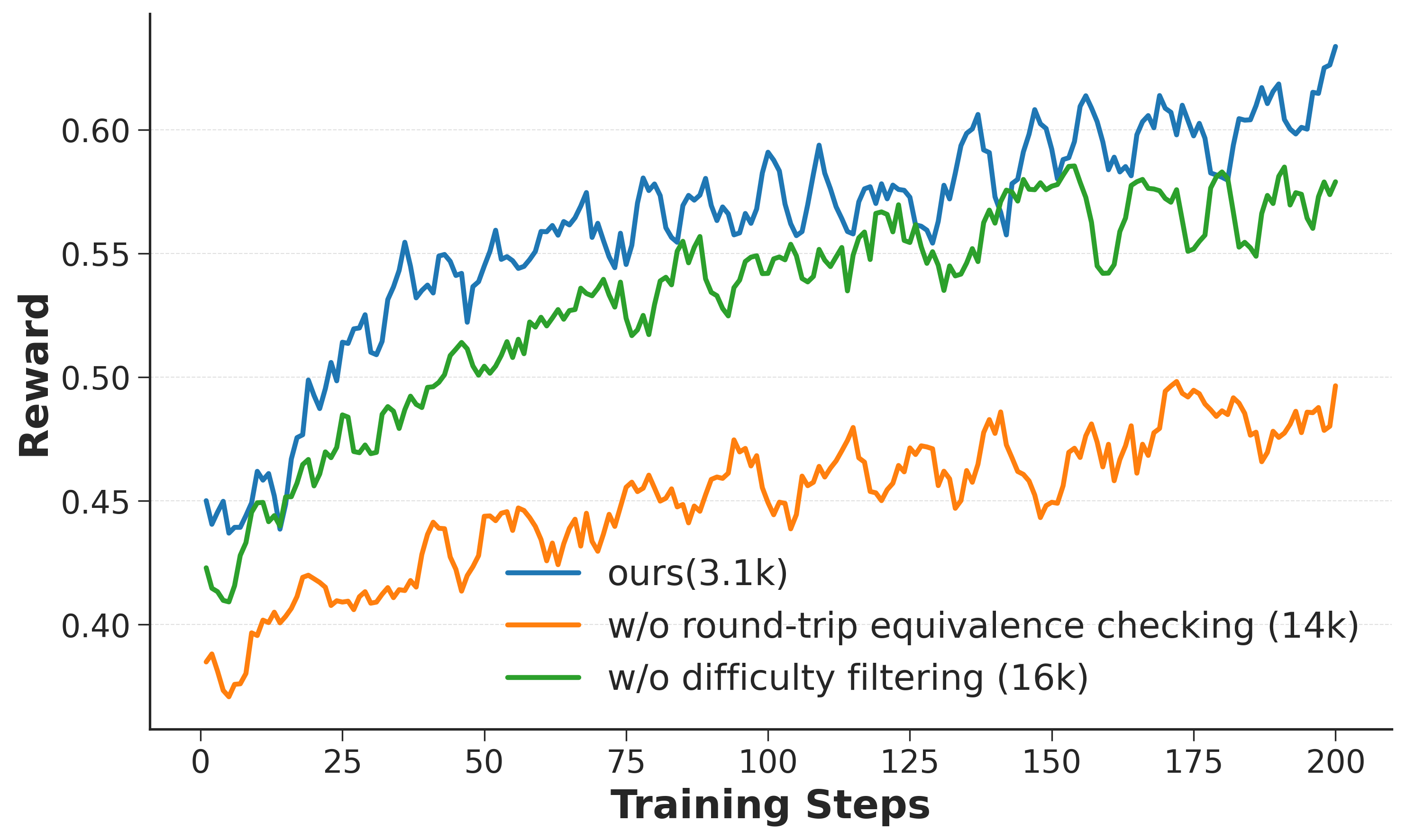}
    \caption{\label{fig:ablation_train_reward}}
  \end{subfigure}

  \caption{\textbf{Train-time scale up on some key metrics.} Figure (a) tracks response length, whereas Figure (b) presents the corresponding trend for reward.}
  \label{fig:ablation_train}
  \vspace{-10pt}
\end{figure}



We perform reinforcement learning using \modelDistill on the three aforementioned datasets, employing identical training parameters. Key metrics observed during these training processes are presented in Figure~\ref{fig:ablation_train}.
Inspection of Figure~\ref{fig:ablation_train_resp_len} reveals distinct trends in response length during training. Utilizing the original RL dataset leads to a noticeable subsequent increase in response length, whereas the training dataset without difficulty filtering leads to a segment of response decrease. This suggests that even when initial responses are relatively long, incorporating more challenging samples during reinforcement learning facilitates further steady growth in response length. Figure~\ref{fig:ablation_train_reward} illustrates that the pseudo-golden dataset consistently exhibits notably lower reward throughout the training process compared to our original RL dataset. This underscores the critical role of golden code accuracy during reinforcement learning.


\subsubsection{Acceleration via Adaptive DAPO}

To quantitatively demonstrate the acceleration achieved by our adaptive DAPO algorithm, we provide a comparison of time usage in Figure~\ref{fig:adaptive_dapo_comparison_all}.
The plots reveal a notable increase in the time per RL step in baseline DAPO training around step 150. This performance degradation in the baseline is attributed to its fixed generation batch size, which becomes insufficient to yield enough samples for a complete training batch as training progresses. 
In contrast, our adaptive DAPO effectively mitigates this issue. It dynamically adjusts and increases the generation batch size across steps. In addition, when a generation attempt does not produce sufficient valid samples for a training batch, the algorithm recalculates the required remaining batch size. 
In Figure~\ref{fig:adaptive_dapo_comparison_avg}, we provide the average speedup of adaptive DAPO, along with a breakdown of performance before and after step 150.
Notably, the time reduction after step 150 is significantly more pronounced—the speedup factor reaches 1.44 after step 150, compared to 1.04 before step 150. This disparity highlights the critical benefit of eliminating sampling more than once.
After applying adaptive DAPO, the final speedup factor reaches 1.25×.

\begin{figure}[h]
  \centering
  \begin{subfigure}{0.4\linewidth}
    \centering
    \includegraphics[width=\linewidth]{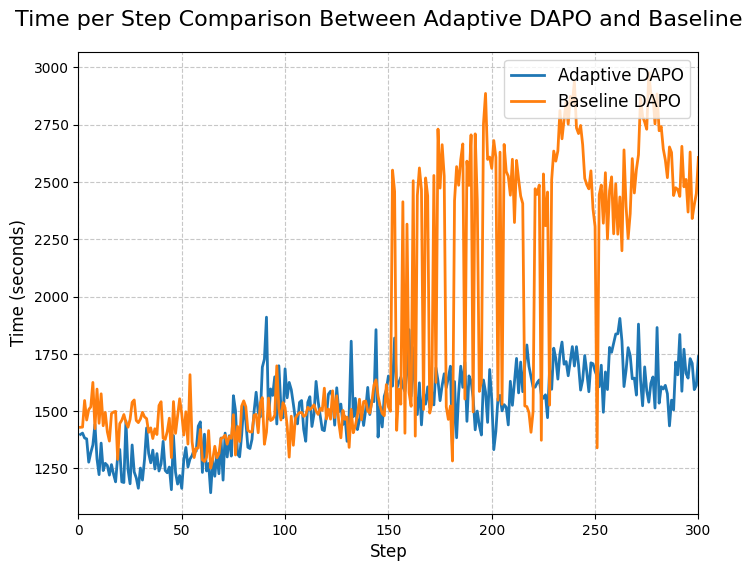}
    \caption{\label{fig:adaptive_dapo_comparison_all}}
  \end{subfigure}
  \hspace{0.1\linewidth} 
  \begin{subfigure}{0.4\linewidth}
    \centering
    \includegraphics[width=\linewidth]{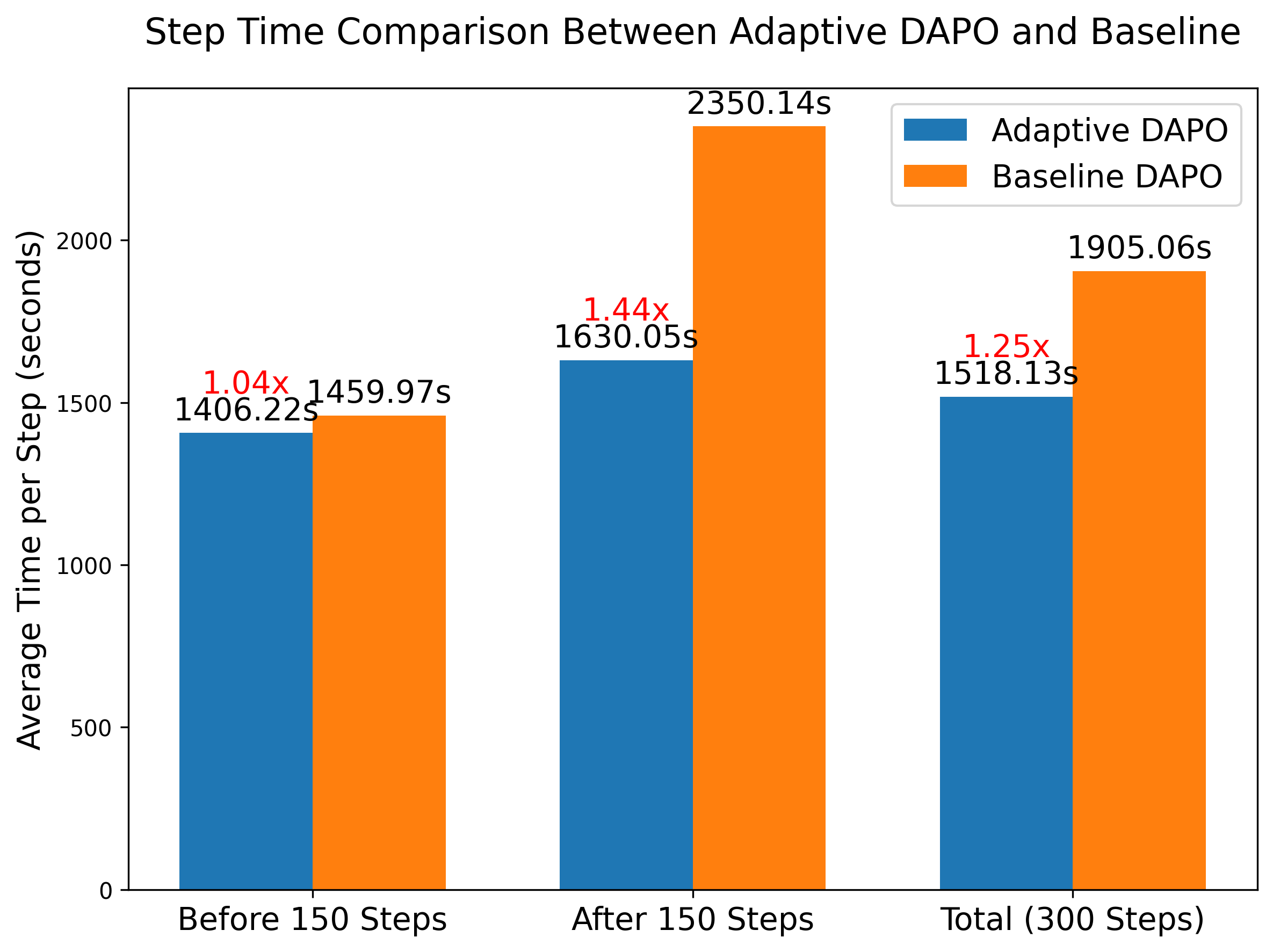}
    \caption{\label{fig:adaptive_dapo_comparison_avg}}
  \end{subfigure}
  \caption{\textbf{Time comparison between adaptive DAPO and baseline DAPO.} \textbf{(a)}: Comparison of RL training time per step. \textbf{(b)}: Acceleration ratio between adaptive DAPO and baseline DAPO, breakdown by step (whether before 150).}
  \label{fig:adaptive_dapo_comparison}
\vspace{-10pt}
\end{figure}



\subsubsection{Testbench Performance Evaluation} \label{sec:testbench}

We evaluate our auto-testbench generation framework against a \DSIII-generated testbench, both taking the Verilog code from GitHub as the golden reference. We conduct two key tests:

\textbf{Correctness classification test.}
We assess whether the testbenches might misclassify correct code as incorrect. To do this, we validate with "golden vs. golden" inputs (i.e., comparing the golden code against itself). The expected outcome is 100\% correct classification. Our method misclassifies only 0.3\% of cases (due to problems of Icarus Verilog simulation and randomization issues) — a 96.1\% reduction in false negatives compared to the \DSIII-generated testbench (7.6\%).

\textbf{Fuzzing test for sequential circuits.}
Second, we perform a fuzzing test on sequential circuits by instructing \DSIII to inject subtle errors into the golden code. The goal is to measure how effectively each testbench detects these mistakes. Our testbench detects 65\% of the injected errors, demonstrating a 62.5\% relative improvement in detection rate over the \DSIII testbench (40\%) and indicating fewer false positives.

\section{Related Work}

\subsection{Large Language Models for Reasoning}

The \OO \cite{openaio1} series is the first closed-source model trained with large-scale reinforcement learning to perform reasoning through CoT. Inspired by its powerful and effective reinforcement learning training paradigm, QwQ \cite{qwq32b}, \RR \cite{deepseekr1}, and Kimi~k1.5 \cite{kimik1.5} have all adopted and improved upon its approach, achieving promising results.
Limited by computational resources, open-source communities have actively explored low-cost approaches to replicate o1-like reasoning models. Some efforts have focused on distilling the powerful closed-source reasoning models \cite{sky_t1_2025, muennighoff2025s1, ye2025limo, huang2024o1replicationjourney, min2024imitate}. 
while others have also explored training reasoning models using reinforcement learning \cite{hu2025openreasonerzero, zeng2025simplerl, deepscaler2025, deepcoder2025, xie2025logicrl, tinyzero, jin2025search, Fin-R1}. 

The main difference between \modelFinal and the aforementioned reasoning models lies in its focus on hardware description language code generation, which poses unique challenges due to \textbf{verification difficulty} and \textbf{limited data quality}. In contrast, prior works primarily specialised in domains such as mathematics, which benefit from easily verifiable numerical outputs and rich open-source datasets.

\subsection{Large Language Models for Verilog Code Generation}

With the development of large language models, specialised code generation models for hardware description languages also receive widespread attention. Many prior works \cite{RTLCoder, codev, pei2024betterv, cui2024origen, liu2024craftrtl} focus on Verilog instruction-tuning data creation \textbf{without a strict correctness evaluation}. Most works have a syntax check in constructing instruction-response pairs: RTLCoder \cite{RTLCoder} and CodeV \cite{codev} add syntax checks when constructing supervised fine-tuning (SFT) datasets with closed-source LLMs. BetterV \cite{pei2024betterv} maps code across languages using Verilog syntax constraints, while OriGen \cite{cui2024origen} leverages compiler feedback to eliminate syntax errors. For functional correctness, to date, only CraftRTL's correct-by-construction approach \cite{liu2024craftrtl} ensures functional correspondence between instruction and response through formal verification. However, its applicability remains restricted to Karnaugh maps and finite-state machines, a narrow subset of Verilog design challenges.

This verification bottleneck shifts to the model optimization stage. Specifically, reinforcement learning with rule-based rewards attempts to address functional correctness by relying on testbenches for reward calculation. However, this strategy is undermined by the fact that current testbench generation paradigms suffer from two systemic flaws:
(1) \textbf{Unverified validation frameworks}:
For example, VeriPrefer \cite{wang2025veriprefer} optimizes testbench coverage, but its testbenches themselves may be flawed, sometimes failing to pass the reference code they were designed to verify.
ReasoningV \cite{qin2025reasoningv} co-generates code and testbenches via DeepSeek-R1, inheriting the model’s hallucination risks.
(2) \textbf{Cost-prohibitive iteration}:
AutoBench \cite{qiu2024autobench} and CorrectBench \cite{qiu2024correctbench} employ multi-stage LLM workflows, where each self-correction cycle incurs escalating computational costs and latency, directly conflicting with RL’s demand for rapid, low-cost reward feedback.

Unlike prior work, we apply Verilog functional verification with auto-generated equivalence checking (see Section \ref{sec:equivalence_checking}), providing a robust foundation for both data curation and reinforcement learning.


\section{Conclusion}

In this paper, we propose \xname, a unified RLVR framework designed for training Verilog generation LLMs. This framework first distills data with reasoning patterns and then applies reinforcement learning on high-quality data curated by an automated testbench generation framework. The model trained via this framework, \modelname, achieves outstanding performance on Verilog generation benchmarks like VerilogEval and RTLLM, matching or even surpassing DeepSeek-R1, which demonstrates the effectiveness of the automated testbench generation and the two-stage training paradigm. 
Notably, to the best of our knowledge, \modelname~is the first Verilog-domain LLM to demonstrate strong reasoning capability showing consistent test-time scaling behavior.
A series of analytical experiments further highlights the powerful impact of \xname~framework in enhancing data quality and further unlocking the Verilog code generation capabilities of LLMs through reasoning.

\FloatBarrier
\newpage

\section*{Acknowledgements}

This work is partially supported by the Strategic Priority Research Program of the Chinese Academy of Sciences (Grants No.XDB0660300, XDB0660301, XDB0660302), the NSF of China (Grants No.62341411, 62222214, 62525203, U22A2028, 6240073476), CAS Project for Young Scientists in Basic Research (YSBR-029) and Youth Innovation Promotion Association CAS.

\bibliography{references}

@misc{codev,
      title={CodeV: Empowering LLMs for Verilog Generation through Multi-Level Summarization}, 
      author={Yang Zhao and Di Huang and Chongxiao Li and Pengwei Jin and Ziyuan Nan and Tianyun Ma and Lei Qi and Yansong Pan and Zhenxing Zhang and Rui Zhang and Xishan Zhang and Zidong Du and Qi Guo and Xing Hu and Yunji Chen},
      year={2024},
      eprint={2407.10424},
      archivePrefix={arXiv},
      primaryClass={cs.PL},
      url={https://arxiv.org/abs/2407.10424}, 
}

@misc{openaio1,
      title={OpenAI o1 System Card}, 
      author={OpenAI},
      year={2024},
      eprint={2412.16720},
      archivePrefix={arXiv},
      primaryClass={cs.AI},
      url={https://arxiv.org/abs/2412.16720}, 
}

@misc{deepseekr1,
      title={DeepSeek-R1: Incentivizing Reasoning Capability in LLMs via Reinforcement Learning}, 
      author={DeepSeek-AI},
      year={2025},
      eprint={2501.12948},
      archivePrefix={arXiv},
      primaryClass={cs.CL},
      url={https://arxiv.org/abs/2501.12948}, 
}

@article{kimik1.5,
  title={Kimi k1. 5: Scaling reinforcement learning with llms},
  author={Team, Kimi and Du, Angang and Gao, Bofei and Xing, Bowei and Jiang, Changjiu and Chen, Cheng and Li, Cheng and Xiao, Chenjun and Du, Chenzhuang and Liao, Chonghua and others},
  journal={arXiv preprint arXiv:2501.12599},
  year={2025}
}

@misc{qwq32b,
    title = {QwQ-32B: Embracing the Power of Reinforcement Learning},
    url = {https://qwenlm.github.io/blog/qwq-32b/},
    author = {Qwen Team},
    month = {March},
    year = {2025}
}

@article{qwen2.5,
      title={Qwen2.5 Technical Report}, 
      author={An Yang and Baosong Yang and Beichen Zhang and Binyuan Hui and Bo Zheng and Bowen Yu and Chengyuan Li and Dayiheng Liu and Fei Huang and Haoran Wei and Huan Lin and Jian Yang and Jianhong Tu and Jianwei Zhang and Jianxin Yang and Jiaxi Yang and Jingren Zhou and Junyang Lin and Kai Dang and Keming Lu and Keqin Bao and Kexin Yang and Le Yu and Mei Li and Mingfeng Xue and Pei Zhang and Qin Zhu and Rui Men and Runji Lin and Tianhao Li and Tianyi Tang and Tingyu Xia and Xingzhang Ren and Xuancheng Ren and Yang Fan and Yang Su and Yichang Zhang and Yu Wan and Yuqiong Liu and Zeyu Cui and Zhenru Zhang and Zihan Qiu},
      journal={arXiv preprint arXiv:2412.15115},
      year={2024}
}

@misc{qwen2.5coder,
      title={Qwen2.5-Coder Technical Report}, 
      author={Binyuan Hui and Jian Yang and Zeyu Cui and Jiaxi Yang and Dayiheng Liu and Lei Zhang and Tianyu Liu and Jiajun Zhang and Bowen Yu and Keming Lu and Kai Dang and Yang Fan and Yichang Zhang and An Yang and Rui Men and Fei Huang and Bo Zheng and Yibo Miao and Shanghaoran Quan and Yunlong Feng and Xingzhang Ren and Xuancheng Ren and Jingren Zhou and Junyang Lin},
      year={2024},
      eprint={2409.12186},
      archivePrefix={arXiv},
      primaryClass={cs.CL},
      url={https://arxiv.org/abs/2409.12186}, 
}

@misc{deepseekv3,
      title={DeepSeek-V3 Technical Report}, 
      author={DeepSeek-AI},
      year={2025},
      eprint={2412.19437},
      archivePrefix={arXiv},
      primaryClass={cs.CL},
      url={https://arxiv.org/abs/2412.19437}, 
}

@misc{openai2024gpt4ocard,
      title={GPT-4o System Card}, 
      author={OpenAI},
      year={2024},
      eprint={2410.21276},
      archivePrefix={arXiv},
      primaryClass={cs.CL},
      url={https://arxiv.org/abs/2410.21276}, 
}

@inproceedings{rouge-l,
  title={ROUGE: A Package for Automatic Evaluation of Summaries},
  author={Chin-Yew Lin},
  booktitle={Annual Meeting of the Association for Computational Linguistics},
  year={2004},
  url={https://api.semanticscholar.org/CorpusID:964287}
}

@inproceedings{yosys,
  title={Yosys-A Free Verilog Synthesis Suite},
  author={Clifford Wolf and Johann Glaser and Johannes Kepler},
  year={2013},
  url={https://api.semanticscholar.org/CorpusID:202611483},
  booktitle={clifford.fm}
}

@article{RTLCoder,
  title={RTLCoder: Fully Open-Source and Efficient LLM-Assisted RTL Code Generation Technique},
  author={Liu, Shang and Fang, Wenji and Lu, Yao and Wang, Jing and Zhang, Qijun and Zhang, Hongce and Xie, Zhiyao},
  journal={IEEE Transactions on Computer-Aided Design of Integrated Circuits and Systems},
  year={2024},
  publisher={IEEE}
}

@inproceedings{liu2023verilogeval,
  title={Verilogeval: Evaluating large language models for verilog code generation},
  author={Liu, Mingjie and Pinckney, Nathaniel and Khailany, Brucek and Ren, Haoxing},
  booktitle={2023 IEEE/ACM International Conference on Computer Aided Design (ICCAD)},
  pages={1--8},
  year={2023},
  organization={IEEE}
}

@misc{verilogeval_v2,
      title={Revisiting VerilogEval: Newer LLMs, In-Context Learning, and Specification-to-RTL Tasks}, 
      author={Nathaniel Pinckney and Christopher Batten and Mingjie Liu and Haoxing Ren and Brucek Khailany},
      year={2024},
      eprint={2408.11053},
      archivePrefix={arXiv},
      primaryClass={cs.SE},
      url={https://arxiv.org/abs/2408.11053}, 
}

@inproceedings{rtllm,
  author={Lu, Yao and Liu, Shang and Zhang, Qijun and Xie, Zhiyao},
  booktitle={2024 29th Asia and South Pacific Design Automation Conference (ASP-DAC)}, 
  title={RTLLM: An Open-Source Benchmark for Design RTL Generation with Large Language Model}, 
  year={2024},
  pages={722-727},
  organization={IEEE}
}

@inproceedings{rtllm_v2,
  title={Openllm-rtl: Open dataset and benchmark for llm-aided design rtl generation},
  author={Liu, Shang and Lu, Yao and Fang, Wenji and Li, Mengming and Xie, Zhiyao},
  booktitle={Proceedings of the 43rd IEEE/ACM International Conference on Computer-Aided Design},
  pages={1--9},
  year={2024}
}

@misc{Fin-R1,
      title={Fin-R1: A Large Language Model for Financial Reasoning through Reinforcement Learning}, 
      author={Zhaowei Liu and Xin Guo and Fangqi Lou and Lingfeng Zeng and Jinyi Niu and Zixuan Wang and Jiajie Xu and Weige Cai and Ziwei Yang and Xueqian Zhao and Chao Li and Sheng Xu and Dezhi Chen and Yun Chen and Zuo Bai and Liwen Zhang},
      year={2025},
      eprint={2503.16252},
      archivePrefix={arXiv},
      primaryClass={cs.CL},
      url={https://arxiv.org/abs/2503.16252}, 
}

@misc{sky_t1_2025,
  author       = {NovaSky Team},
  title        = {Sky-T1: Train your own O1 preview model within \$450},
  howpublished = {https://novasky-ai.github.io/posts/sky-t1},
  note         = {Accessed: 2025-01-09},
  year         = {2025}
}

@article{muennighoff2025s1,
  title={s1: Simple test-time scaling},
  author={Muennighoff, Niklas and Yang, Zitong and Shi, Weijia and Li, Xiang Lisa and Fei-Fei, Li and Hajishirzi, Hannaneh and Zettlemoyer, Luke and Liang, Percy and Cand{\`e}s, Emmanuel and Hashimoto, Tatsunori},
  journal={arXiv preprint arXiv:2501.19393},
  year={2025}
}

@article{ye2025limo,
  title={LIMO: Less is More for Reasoning},
  author={Ye, Yixin and Huang, Zhen and Xiao, Yang and Chern, Ethan and Xia, Shijie and Liu, Pengfei},
  journal={arXiv preprint arXiv:2502.03387},
  year={2025}
}

@misc{huang2024o1replicationjourney,
      title={O1 Replication Journey -- Part 2: Surpassing O1-preview through Simple Distillation, Big Progress or Bitter Lesson?}, 
      author={Zhen Huang and Haoyang Zou and Xuefeng Li and Yixiu Liu and Yuxiang Zheng and Ethan Chern and Shijie Xia and Yiwei Qin and Weizhe Yuan and Pengfei Liu},
      year={2024},
      eprint={2411.16489},
      archivePrefix={arXiv},
      primaryClass={cs.CL},
      url={https://arxiv.org/abs/2411.16489}, 
}

@misc{min2024imitate,
      title={Imitate, Explore, and Self-Improve: A Reproduction Report on Slow-thinking Reasoning Systems}, 
      author={Yingqian Min and Zhipeng Chen and Jinhao Jiang and Jie Chen and Jia Deng and Yiwen Hu and Yiru Tang and Jiapeng Wang and Xiaoxue Cheng and Huatong Song and Wayne Xin Zhao and Zheng Liu and Zhongyuan Wang and Ji-Rong Wen},
      year={2024},
      eprint={2412.09413},
      archivePrefix={arXiv},
      primaryClass={cs.AI},
      url={https://arxiv.org/abs/2412.09413}, 
}

@article{yu2025dapo,
  title={Dapo: An open-source llm reinforcement learning system at scale},
  author={Yu, Qiying and Zhang, Zheng and Zhu, Ruofei and Yuan, Yufeng and Zuo, Xiaochen and Yue, Yu and Fan, Tiantian and Liu, Gaohong and Liu, Lingjun and Liu, Xin and others},
  journal={arXiv preprint arXiv:2503.14476},
  year={2025}
}

@misc{zeng2025simplerl,
      title={SimpleRL-Zoo: Investigating and Taming Zero Reinforcement Learning for Open Base Models in the Wild}, 
      author={Weihao Zeng and Yuzhen Huang and Qian Liu and Wei Liu and Keqing He and Zejun Ma and Junxian He},
      year={2025},
      eprint={2503.18892},
      archivePrefix={arXiv},
      primaryClass={cs.LG},
      url={https://arxiv.org/abs/2503.18892}, 
}

@misc{hu2025openreasonerzero,
      title={Open-Reasoner-Zero: An Open Source Approach to Scaling Up Reinforcement Learning on the Base Model}, 
      author={Jingcheng Hu and Yinmin Zhang and Qi Han and Daxin Jiang and Xiangyu Zhang and Heung-Yeung Shum},
      year={2025},
      eprint={2503.24290},
      archivePrefix={arXiv},
      primaryClass={cs.LG},
      url={https://arxiv.org/abs/2503.24290}, 
}

@misc{deepscaler2025,
  title={DeepScaleR: Surpassing O1-Preview with a 1.5B Model by Scaling RL},
  author={Michael Luo and Sijun Tan and Justin Wong and Xiaoxiang Shi and William Y. Tang and Manan Roongta and Colin Cai and Jeffrey Luo and Li Erran Li and Raluca Ada Popa and Ion Stoica},
  howpublished={\url{https://pretty-radio-b75.notion.site/DeepScaleR-Surpassing-O1-Preview-with-a-1-5B-Model-by-Scaling-RL}},
  note={Notion Blog},
  year={2025}
}

@misc{deepcoder2025,
  title={DeepCoder: A Fully Open-Source 14B Coder at O3-mini Level},
  author={Michael Luo and Sijun Tan and Roy Huang and Ameen Patel and Alpay Ariyak and Qingyang Wu and Xiaoxiang Shi and Rachel Xin and Colin Cai and Maurice Weber and Ce Zhang and Li Erran Li and Raluca Ada Popa and Ion Stoica},
  howpublished={\url{https://pretty-radio-b75.notion.site/DeepCoder-A-Fully-Open-Source-14B-Coder-at-O3-mini-Level}},
  note={Notion Blog},
  year={2025}
}

@misc{xie2025logicrl,
      title={Logic-RL: Unleashing LLM Reasoning with Rule-Based Reinforcement Learning}, 
      author={Tian Xie and Zitian Gao and Qingnan Ren and Haoming Luo and Yuqian Hong and Bryan Dai and Joey Zhou and Kai Qiu and Zhirong Wu and Chong Luo},
      year={2025},
      eprint={2502.14768},
      archivePrefix={arXiv},
      primaryClass={cs.CL},
      url={https://arxiv.org/abs/2502.14768}, 
}

@misc{tinyzero,
    author       = {Jiayi Pan and Junjie Zhang and Xingyao Wang and Lifan Yuan and Hao Peng and Alane Suhr},
    title        = {TinyZero},
    howpublished = {https://github.com/Jiayi-Pan/TinyZero},
    note         = {Accessed: 2025-01-24},
    year         = {2025}
}

@article{jin2025search,
  title={Search-R1: Training LLMs to Reason and Leverage Search Engines with Reinforcement Learning},
  author={Jin, Bowen and Zeng, Hansi and Yue, Zhenrui and Wang, Dong and Zamani, Hamed and Han, Jiawei},
  journal={arXiv preprint arXiv:2503.09516},
  year={2025}
}

@article{hendrycks2021apps,
  title={Measuring coding challenge competence with apps},
  author={Hendrycks, Dan and Basart, Steven and Kadavath, Saurav and Mazeika, Mantas and Arora, Akul and Guo, Ethan and Burns, Collin and Puranik, Samir and He, Horace and Song, Dawn and others},
  journal={arXiv preprint arXiv:2105.09938},
  year={2021}
}

@article{li2023taco,
  title={Taco: Topics in algorithmic code generation dataset},
  author={Li, Rongao and Fu, Jie and Zhang, Bo-Wen and Huang, Tao and Sun, Zhihong and Lyu, Chen and Liu, Guang and Jin, Zhi and Li, Ge},
  journal={arXiv preprint arXiv:2312.14852},
  year={2023}
}

@article{xu2025kodcode,
  title={Kodcode: A diverse, challenging, and verifiable synthetic dataset for coding},
  author={Xu, Zhangchen and Liu, Yang and Yin, Yueqin and Zhou, Mingyuan and Poovendran, Radha},
  journal={arXiv preprint arXiv:2503.02951},
  year={2025}
}

@inproceedings{pei2024betterv,
  title={Betterv: Controlled verilog generation with discriminative guidance},
  author={Zehua, PEI and Zhen, Huiling and Yuan, Mingxuan and Huang, Yu and Yu, Bei},
  booktitle={Forty-first International Conference on Machine Learning},
  year={2024}
}

@inproceedings{cui2024origen,
  title={Origen: Enhancing rtl code generation with code-to-code augmentation and self-reflection},
  author={Cui, Fan and Yin, Chenyang and Zhou, Kexing and Xiao, Youwei and Sun, Guangyu and Xu, Qiang and Guo, Qipeng and Liang, Yun and Zhang, Xingcheng and Song, Demin and others},
  booktitle={Proceedings of the 43rd IEEE/ACM International Conference on Computer-Aided Design},
  pages={1--9},
  year={2024}
}

@article{liu2024craftrtl,
  title={CraftRTL: High-quality Synthetic Data Generation for Verilog Code Models with Correct-by-Construction Non-Textual Representations and Targeted Code Repair},
  author={Liu, Mingjie and Tsai, Yun-Da and Zhou, Wenfei and Ren, Haoxing},
  journal={arXiv preprint arXiv:2409.12993},
  year={2024}
}

@inproceedings{gao2024autovcoder,
  title={AutoVCoder: A Systematic Framework for Automated Verilog Code Generation using LLMs},
  author={Gao, Mingzhe and Zhao, Jieru and Lin, Zhe and Ding, Wenchao and Hou, Xiaofeng and Feng, Yu and Li, Chao and Guo, Minyi},
  booktitle={2024 IEEE 42nd International Conference on Computer Design (ICCD)},
  pages={162--169},
  year={2024},
  organization={IEEE}
}

@misc{wang2025veriprefer,
      title={Insights from Verification: Training a Verilog Generation LLM with Reinforcement Learning with Testbench Feedback}, 
      author={Ning Wang and Bingkun Yao and Jie Zhou and Yuchen Hu and Xi Wang and Nan Guan and Zhe Jiang},
      year={2025},
      eprint={2504.15804},
      archivePrefix={arXiv},
      primaryClass={cs.AR},
      url={https://arxiv.org/abs/2504.15804}, 
}

@misc{qin2025reasoningv,
      title={ReasoningV: Efficient Verilog Code Generation with Adaptive Hybrid Reasoning Model}, 
      author={Haiyan Qin and Zhiwei Xie and Jingjing Li and Liangchen Li and Xiaotong Feng and Junzhan Liu and Wang Kang},
      year={2025},
      eprint={2504.14560},
      archivePrefix={arXiv},
      primaryClass={cs.AR},
      url={https://arxiv.org/abs/2504.14560}, 
}

@inproceedings{qiu2024autobench,
  title={Autobench: Automatic testbench generation and evaluation using llms for hdl design},
  author={Qiu, Ruidi and Zhang, Grace Li and Drechsler, Rolf and Schlichtmann, Ulf and Li, Bing},
  booktitle={Proceedings of the 2024 ACM/IEEE International Symposium on Machine Learning for CAD},
  pages={1--10},
  year={2024}
}

@misc{qiu2024correctbench,
      title={CorrectBench: Automatic Testbench Generation with Functional Self-Correction using LLMs for HDL Design}, 
      author={Ruidi Qiu and Grace Li Zhang and Rolf Drechsler and Ulf Schlichtmann and Bing Li},
      year={2024},
      eprint={2411.08510},
      archivePrefix={arXiv},
      primaryClass={cs.SE},
      url={https://arxiv.org/abs/2411.08510}, 
}

@inproceedings{zheng2024llamafactory,
  title={LlamaFactory: Unified Efficient Fine-Tuning of 100+ Language Models},
  author={Zheng, Yaowei and Zhang, Richong and Zhang, Junhao and Ye, Yanhan and Luo, Zheyan},
  booktitle={Proceedings of the 62nd Annual Meeting of the Association for Computational Linguistics (Volume 3: System Demonstrations)},
  pages={400--410},
  year={2024}
}

@article{sheng2024hybridflow,
  title   = {HybridFlow: A Flexible and Efficient RLHF Framework},
  author  = {Guangming Sheng and Chi Zhang and Zilingfeng Ye and Xibin Wu and Wang Zhang and Ru Zhang and Yanghua Peng and Haibin Lin and Chuan Wu},
  year    = {2024},
  journal = {arXiv preprint arXiv: 2409.19256}
}

@misc{wang2025largelanguagemodelverilog,
      title={Large Language Model for Verilog Generation with Code-Structure-Guided Reinforcement Learning}, 
      author={Ning Wang and Bingkun Yao and Jie Zhou and Xi Wang and Zhe Jiang and Nan Guan},
      year={2025},
      eprint={2407.18271},
      archivePrefix={arXiv},
      primaryClass={cs.AR},
      url={https://arxiv.org/abs/2407.18271}, 
}

@article{
alphacode,
author = {Yujia Li  and David Choi  and Junyoung Chung  and Nate Kushman  and Julian Schrittwieser  and Rémi Leblond  and Tom Eccles  and James Keeling  and Felix Gimeno  and Agustin Dal Lago  and Thomas Hubert  and Peter Choy  and Cyprien de Masson d’Autume  and Igor Babuschkin  and Xinyun Chen  and Po-Sen Huang  and Johannes Welbl  and Sven Gowal  and Alexey Cherepanov  and James Molloy  and Daniel J. Mankowitz  and Esme Sutherland Robson  and Pushmeet Kohli  and Nando de Freitas  and Koray Kavukcuoglu  and Oriol Vinyals },
title = {Competition-level code generation with AlphaCode},
journal = {Science},
volume = {378},
number = {6624},
pages = {1092-1097},
year = {2022},
doi = {10.1126/science.abq1158},
URL = {https://www.science.org/doi/abs/10.1126/science.abq1158},
eprint = {https://www.science.org/doi/pdf/10.1126/science.abq1158}}

@inproceedings{mgverilog,
  title={{MG-Verilog:} Multi-grained Dataset Towards Enhanced LLM-assisted Verilog Generation},
  author={Zhang, Yongan and Yu, Zhongzhi and Fu, Yonggan and Wan, Cheng and Lin, Yingyan (Celine)},
  booktitle={The First IEEE International Workshop on LLM-Aided Design (LAD'24)}, 
  year={2024}
}

@article{chen2021evaluating,
  title={Evaluating large language models trained on code},
  author={Chen, Mark and Tworek, Jerry and Jun, Heewoo and Yuan, Qiming and Pinto, Henrique Ponde De Oliveira and Kaplan, Jared and Edwards, Harri and Burda, Yuri and Joseph, Nicholas and Brockman, Greg and others},
  journal={arXiv preprint arXiv:2107.03374},
  year={2021}
}

@article{su2022one,
  title={One embedder, any task: Instruction-finetuned text embeddings},
  author={Su, Hongjin and Shi, Weijia and Kasai, Jungo and Wang, Yizhong and Hu, Yushi and Ostendorf, Mari and Yih, Wen-tau and Smith, Noah A and Zettlemoyer, Luke and Yu, Tao},
  journal={arXiv preprint arXiv:2212.09741},
  year={2022}
}

@article{maaten2008visualizing,
  title={Visualizing data using t-SNE},
  author={Maaten, Laurens van der and Hinton, Geoffrey},
  journal={Journal of machine learning research},
  volume={9},
  number={Nov},
  pages={2579--2605},
  year={2008}
}

@inproceedings{zhao2025mage,
  title={Mage: A multi-agent engine for automated rtl code generation},
  author={Zhao, Yujie and Zhang, Hejia and Huang, Hanxian and Yu, Zhongming and Zhao, Jishen},
  booktitle={2025 62nd ACM/IEEE Design Automation Conference (DAC)},
  pages={1--7},
  year={2025},
  organization={IEEE}
}
\bibliographystyle{plainnat}
\FloatBarrier




\newpage

\newpage
\section*{NeurIPS Paper Checklist}

\begin{enumerate}

\item {\bf Claims}
    \item[] Question: Do the main claims made in the abstract and introduction accurately reflect the paper's contributions and scope?
    \item[] Answer: \answerYes{} 
    \item[] Justification: The abstract and introduction accurately describe the paper's contributions, including the equivalence checking, dataset filtering based on this equivalence checking, and adaptive DAPO. These are detailed in Section ~\ref{sec:introduction} (Introduction) and Section~\ref{sec:method} (Methods).
    \item[] Guidelines:
    \begin{itemize}
        \item The answer NA means that the abstract and introduction do not include the claims made in the paper.
        \item The abstract and/or introduction should clearly state the claims made, including the contributions made in the paper and important assumptions and limitations. A No or NA answer to this question will not be perceived well by the reviewers. 
        \item The claims made should match theoretical and experimental results, and reflect how much the results can be expected to generalize to other settings. 
        \item It is fine to include aspirational goals as motivation as long as it is clear that these goals are not attained by the paper. 
    \end{itemize}

\item {\bf Limitations}
    \item[] Question: Does the paper discuss the limitations of the work performed by the authors?
    \item[] Answer: \answerYes{} 
    \item[] Justification: The limitations of this work primarily include the cost of data synthesis and potential issues with low-quality data, which are discussed in the Appendix.
    \item[] Guidelines:
    \begin{itemize}
        \item The answer NA means that the paper has no limitation while the answer No means that the paper has limitations, but those are not discussed in the paper. 
        \item The authors are encouraged to create a separate "Limitations" section in their paper.
        \item The paper should point out any strong assumptions and how robust the results are to violations of these assumptions (e.g., independence assumptions, noiseless settings, model well-specification, asymptotic approximations only holding locally). The authors should reflect on how these assumptions might be violated in practice and what the implications would be.
        \item The authors should reflect on the scope of the claims made, e.g., if the approach was only tested on a few datasets or with a few runs. In general, empirical results often depend on implicit assumptions, which should be articulated.
        \item The authors should reflect on the factors that influence the performance of the approach. For example, a facial recognition algorithm may perform poorly when image resolution is low or images are taken in low lighting. Or a speech-to-text system might not be used reliably to provide closed captions for online lectures because it fails to handle technical jargon.
        \item The authors should discuss the computational efficiency of the proposed algorithms and how they scale with dataset size.
        \item If applicable, the authors should discuss possible limitations of their approach to address problems of privacy and fairness.
        \item While the authors might fear that complete honesty about limitations might be used by reviewers as grounds for rejection, a worse outcome might be that reviewers discover limitations that aren't acknowledged in the paper. The authors should use their best judgment and recognize that individual actions in favor of transparency play an important role in developing norms that preserve the integrity of the community. Reviewers will be specifically instructed to not penalize honesty concerning limitations.
    \end{itemize}

\item {\bf Theory assumptions and proofs}
    \item[] Question: For each theoretical result, does the paper provide the full set of assumptions and a complete (and correct) proof?
    \item[] Answer: \answerYes{} 
    \item[] Justification: Theorem 2.1 (Semantic Equivalence in Chain Transformations) is provided under the NLCDE assumption (Definition 2.1).
    \item[] Guidelines:
    \begin{itemize}
        \item The answer NA means that the paper does not include theoretical results. 
        \item All the theorems, formulas, and proofs in the paper should be numbered and cross-referenced.
        \item All assumptions should be clearly stated or referenced in the statement of any theorems.
        \item The proofs can either appear in the main paper or the supplemental material, but if they appear in the supplemental material, the authors are encouraged to provide a short proof sketch to provide intuition. 
        \item Inversely, any informal proof provided in the core of the paper should be complemented by formal proofs provided in appendix or supplemental material.
        \item Theorems and Lemmas that the proof relies upon should be properly referenced. 
    \end{itemize}

    \item {\bf Experimental result reproducibility}
    \item[] Question: Does the paper fully disclose all the information needed to reproduce the main experimental results of the paper to the extent that it affects the main claims and/or conclusions of the paper (regardless of whether the code and data are provided or not)?
    \item[] Answer: \answerYes{} 
    \item[] The paper describes the dataset creation process, RL algorithm, evaluation metrics, LLMs used, experimental setup including SFT, RL, and evaluation parameters (Sections~\ref{sec:method} and~\ref{sec:experiments}).
    \item[] Guidelines:
    \begin{itemize}
        \item The answer NA means that the paper does not include experiments.
        \item If the paper includes experiments, a No answer to this question will not be perceived well by the reviewers: Making the paper reproducible is important, regardless of whether the code and data are provided or not.
        \item If the contribution is a dataset and/or model, the authors should describe the steps taken to make their results reproducible or verifiable. 
        \item Depending on the contribution, reproducibility can be accomplished in various ways. For example, if the contribution is a novel architecture, describing the architecture fully might suffice, or if the contribution is a specific model and empirical evaluation, it may be necessary to either make it possible for others to replicate the model with the same dataset, or provide access to the model. In general. releasing code and data is often one good way to accomplish this, but reproducibility can also be provided via detailed instructions for how to replicate the results, access to a hosted model (e.g., in the case of a large language model), releasing of a model checkpoint, or other means that are appropriate to the research performed.
        \item While NeurIPS does not require releasing code, the conference does require all submissions to provide some reasonable avenue for reproducibility, which may depend on the nature of the contribution. For example
        \begin{enumerate}
            \item If the contribution is primarily a new algorithm, the paper should make it clear how to reproduce that algorithm.
            \item If the contribution is primarily a new model architecture, the paper should describe the architecture clearly and fully.
            \item If the contribution is a new model (e.g., a large language model), then there should either be a way to access this model for reproducing the results or a way to reproduce the model (e.g., with an open-source dataset or instructions for how to construct the dataset).
            \item We recognize that reproducibility may be tricky in some cases, in which case authors are welcome to describe the particular way they provide for reproducibility. In the case of closed-source models, it may be that access to the model is limited in some way (e.g., to registered users), but it should be possible for other researchers to have some path to reproducing or verifying the results.
        \end{enumerate}
    \end{itemize}

\item {\bf Open access to data and code}
    \item[] Question: Does the paper provide open access to the data and code, with sufficient instructions to faithfully reproduce the main experimental results, as described in supplemental material?
    \item[] Answer: \answerYes{} 
    \item[] Justification: We have released our complete model, code, and dataset on Huggingface (model/dataset) and Github (code).
    \item[] Guidelines:
    \begin{itemize}
        \item The answer NA means that paper does not include experiments requiring code.
        \item Please see the NeurIPS code and data submission guidelines (\url{https://nips.cc/public/guides/CodeSubmissionPolicy}) for more details.
        \item While we encourage the release of code and data, we understand that this might not be possible, so “No” is an acceptable answer. Papers cannot be rejected simply for not including code, unless this is central to the contribution (e.g., for a new open-source benchmark).
        \item The instructions should contain the exact command and environment needed to run to reproduce the results. See the NeurIPS code and data submission guidelines (\url{https://nips.cc/public/guides/CodeSubmissionPolicy}) for more details.
        \item The authors should provide instructions on data access and preparation, including how to access the raw data, preprocessed data, intermediate data, and generated data, etc.
        \item The authors should provide scripts to reproduce all experimental results for the new proposed method and baselines. If only a subset of experiments are reproducible, they should state which ones are omitted from the script and why.
        \item At submission time, to preserve anonymity, the authors should release anonymized versions (if applicable).
        \item Providing as much information as possible in supplemental material (appended to the paper) is recommended, but including URLs to data and code is permitted.
    \end{itemize}

\item {\bf Experimental setting/details}
    \item[] Question: Does the paper specify all the training and test details (e.g., data splits, hyperparameters, how they were chosen, type of optimizer, etc.) necessary to understand the results?
    \item[] Answer: \answerYes{} 
    \item[] Justification: Section \ref{sec:implementation_details} (Experiment Implementation Details) has introduced key parameters of SFT, RL, and testing. Besides, the full parameter setting is provided in Appendix B.
    \item[] Guidelines:
    \begin{itemize}
        \item The answer NA means that the paper does not include experiments.
        \item The experimental setting should be presented in the core of the paper to a level of detail that is necessary to appreciate the results and make sense of them.
        \item The full details can be provided either with the code, in appendix, or as supplemental material.
    \end{itemize}

\item {\bf Experiment statistical significance}
    \item[] Question: Does the paper report error bars suitably and correctly defined or other appropriate information about the statistical significance of the experiments?
    \item[] Answer: \answerNo{} 
    \item[] Justification: The paper reports performance metrics as percentages and raw counts in tables such as Table 1, 2. However, it does not include error bars, confidence intervals, or formal statistical significance tests for these results. Conducting multiple runs for each experimental condition to gather data for robust statistical significance analysis was deemed prohibitively expensive. This is primarily due to the high cost of RL training.
    \item[] Guidelines:
    \begin{itemize}
        \item The answer NA means that the paper does not include experiments.
        \item The authors should answer "Yes" if the results are accompanied by error bars, confidence intervals, or statistical significance tests, at least for the experiments that support the main claims of the paper.
        \item The factors of variability that the error bars are capturing should be clearly stated (for example, train/test split, initialization, random drawing of some parameter, or overall run with given experimental conditions).
        \item The method for calculating the error bars should be explained (closed form formula, call to a library function, bootstrap, etc.)
        \item The assumptions made should be given (e.g., Normally distributed errors).
        \item It should be clear whether the error bar is the standard deviation or the standard error of the mean.
        \item It is OK to report 1-sigma error bars, but one should state it. The authors should preferably report a 2-sigma error bar than state that they have a 96\% CI, if the hypothesis of Normality of errors is not verified.
        \item For asymmetric distributions, the authors should be careful not to show in tables or figures symmetric error bars that would yield results that are out of range (e.g. negative error rates).
        \item If error bars are reported in tables or plots, The authors should explain in the text how they were calculated and reference the corresponding figures or tables in the text.
    \end{itemize}

\item {\bf Experiments compute resources}
    \item[] Question: For each experiment, does the paper provide sufficient information on the computer resources (type of compute workers, memory, time of execution) needed to reproduce the experiments?
    \item[] Answer: \answerYes{} 
    \item[] Justification: This paper claims computer resources for SFT and RL in section \ref{sec:implementation_details}.
    \item[] Guidelines:
    \begin{itemize}
        \item The answer NA means that the paper does not include experiments.
        \item The paper should indicate the type of compute workers CPU or GPU, internal cluster, or cloud provider, including relevant memory and storage.
        \item The paper should provide the amount of compute required for each of the individual experimental runs as well as estimate the total compute. 
        \item The paper should disclose whether the full research project required more compute than the experiments reported in the paper (e.g., preliminary or failed experiments that didn't make it into the paper). 
    \end{itemize}
    
\item {\bf Code of ethics}
    \item[] Question: Does the research conducted in the paper conform, in every respect, with the NeurIPS Code of Ethics \url{https://neurips.cc/public/EthicsGuidelines}?
    \item[] Answer: \answerYes{} 
    \item[] Justification: The research focuses on LLM-based Verilog code generation, using publicly available code from GitHub to create a training dataset. It does not involve human subjects, sensitive personal data, or applications with immediate high-risk ethical concerns outlined in the NeurIPS Code of Ethics. The goal is to advance Verilog auto-generation.
    \item[] Guidelines:
    \begin{itemize}
        \item The answer NA means that the authors have not reviewed the NeurIPS Code of Ethics.
        \item If the authors answer No, they should explain the special circumstances that require a deviation from the Code of Ethics.
        \item The authors should make sure to preserve anonymity (e.g., if there is a special consideration due to laws or regulations in their jurisdiction).
    \end{itemize}

\item {\bf Broader impacts}
    \item[] Question: Does the paper discuss both potential positive societal impacts and negative societal impacts of the work performed?
    \item[] Answer: \answerYes{} 
    \item[] Justification: See Introduction and Appendix.
    \item[] Guidelines:
    \begin{itemize}
        \item The answer NA means that there is no societal impact of the work performed.
        \item If the authors answer NA or No, they should explain why their work has no societal impact or why the paper does not address societal impact.
        \item Examples of negative societal impacts include potential malicious or unintended uses (e.g., disinformation, generating fake profiles, surveillance), fairness considerations (e.g., deployment of technologies that could make decisions that unfairly impact specific groups), privacy considerations, and security considerations.
        \item The conference expects that many papers will be foundational research and not tied to particular applications, let alone deployments. However, if there is a direct path to any negative applications, the authors should point it out. For example, it is legitimate to point out that an improvement in the quality of generative models could be used to generate deepfakes for disinformation. On the other hand, it is not needed to point out that a generic algorithm for optimizing neural networks could enable people to train models that generate Deepfakes faster.
        \item The authors should consider possible harms that could arise when the technology is being used as intended and functioning correctly, harms that could arise when the technology is being used as intended but gives incorrect results, and harms following from (intentional or unintentional) misuse of the technology.
        \item If there are negative societal impacts, the authors could also discuss possible mitigation strategies (e.g., gated release of models, providing defenses in addition to attacks, mechanisms for monitoring misuse, mechanisms to monitor how a system learns from feedback over time, improving the efficiency and accessibility of ML).
    \end{itemize}
    
\item {\bf Safeguards}
    \item[] Question: Does the paper describe safeguards that have been put in place for responsible release of data or models that have a high risk for misuse (e.g., pretrained language models, image generators, or scraped datasets)?
    \item[] Answer: \answerNA{} 
    \item[] Justification: The training data used in this work comes from open-source code on GitHub and the open-source language model (e.g., DeepSeek-R1), with no risk of sensitive content leakage.
    \item[] Guidelines:
    \begin{itemize}
        \item The answer NA means that the paper poses no such risks.
        \item Released models that have a high risk for misuse or dual-use should be released with necessary safeguards to allow for controlled use of the model, for example by requiring that users adhere to usage guidelines or restrictions to access the model or implementing safety filters. 
        \item Datasets that have been scraped from the Internet could pose safety risks. The authors should describe how they avoided releasing unsafe images.
        \item We recognize that providing effective safeguards is challenging, and many papers do not require this, but we encourage authors to take this into account and make a best faith effort.
    \end{itemize}

\item {\bf Licenses for existing assets}
    \item[] Question: Are the creators or original owners of assets (e.g., code, data, models), used in the paper, properly credited and are the license and terms of use explicitly mentioned and properly respected?
    \item[] Answer: \answerYes{} 
    \item[] Justification: The paper properly credits the creators of the datasets used: VerilogEval v1 \cite{liu2023verilogeval} / v2 \cite{verilogeval_v2} and RTLLM v1.1 \cite{rtllm} / v2 \cite{rtllm_v2} are cited in Section~\ref{sec:experiments}. Also the creators of the models used: \RR \cite{deepseekr1}, \DSIII \cite{deepseekv3} and Qwen \cite{qwen2.5coder} are cited in Sections \ref{sec:introduction} and \ref{sec:method}. According to public sources, all datasets and models are released under open-source licenses, and their use in this paper complies with the respective licensing terms.
    \item[] Guidelines:
    \begin{itemize}
        \item The answer NA means that the paper does not use existing assets.
        \item The authors should cite the original paper that produced the code package or dataset.
        \item The authors should state which version of the asset is used and, if possible, include a URL.
        \item The name of the license (e.g., CC-BY 4.0) should be included for each asset.
        \item For scraped data from a particular source (e.g., website), the copyright and terms of service of that source should be provided.
        \item If assets are released, the license, copyright information, and terms of use in the package should be provided. For popular datasets, \url{paperswithcode.com/datasets} has curated licenses for some datasets. Their licensing guide can help determine the license of a dataset.
        \item For existing datasets that are re-packaged, both the original license and the license of the derived asset (if it has changed) should be provided.
        \item If this information is not available online, the authors are encouraged to reach out to the asset's creators.
    \end{itemize}

\item {\bf New assets}
    \item[] Question: Are new assets introduced in the paper well documented and is the documentation provided alongside the assets?
    \item[] Answer: \answerYes{} 
    \item[] Justification: The paper introduces the \xname~training framework, and the model trained by it, \modelname~in Section \ref{sec:method}. The paper itself serves as the primary documentation for these new assets. If these assets were to be released, further practical documentation (e.g., README files) would typically accompany them.
    \item[] Guidelines:
    \begin{itemize}
        \item The answer NA means that the paper does not release new assets.
        \item Researchers should communicate the details of the dataset/code/model as part of their submissions via structured templates. This includes details about training, license, limitations, etc. 
        \item The paper should discuss whether and how consent was obtained from people whose asset is used.
        \item At submission time, remember to anonymize your assets (if applicable). You can either create an anonymized URL or include an anonymized zip file.
    \end{itemize}

\item {\bf Crowdsourcing and research with human subjects}
    \item[] Question: For crowdsourcing experiments and research with human subjects, does the paper include the full text of instructions given to participants and screenshots, if applicable, as well as details about compensation (if any)? 
    \item[] Answer: \answerNA{} 
    \item[] Justification: The research described in the paper does not involve crowdsourcing experiments or research with human subjects.
    \item[] Guidelines:
    \begin{itemize}
        \item The answer NA means that the paper does not involve crowdsourcing nor research with human subjects.
        \item Including this information in the supplemental material is fine, but if the main contribution of the paper involves human subjects, then as much detail as possible should be included in the main paper. 
        \item According to the NeurIPS Code of Ethics, workers involved in data collection, curation, or other labor should be paid at least the minimum wage in the country of the data collector. 
    \end{itemize}

\item {\bf Institutional review board (IRB) approvals or equivalent for research with human subjects}
    \item[] Question: Does the paper describe potential risks incurred by study participants, whether such risks were disclosed to the subjects, and whether Institutional Review Board (IRB) approvals (or an equivalent approval/review based on the requirements of your country or institution) were obtained?
    \item[] Answer: \answerNA{} 
    \item[] Justification: The research does not involve human subjects and therefore did not require IRB approval or an equivalent review.
    \item[] Guidelines:
    \begin{itemize}
        \item The answer NA means that the paper does not involve crowdsourcing nor research with human subjects.
        \item Depending on the country in which research is conducted, IRB approval (or equivalent) may be required for any human subjects research. If you obtained IRB approval, you should clearly state this in the paper. 
        \item We recognize that the procedures for this may vary significantly between institutions and locations, and we expect authors to adhere to the NeurIPS Code of Ethics and the guidelines for their institution. 
        \item For initial submissions, do not include any information that would break anonymity (if applicable), such as the institution conducting the review.
    \end{itemize}

\item {\bf Declaration of LLM usage}
    \item[] Question: Does the paper describe the usage of LLMs if it is an important, original, or non-standard component of the core methods in this research? Note that if the LLM is used only for writing, editing, or formatting purposes and does not impact the core methodology, scientific rigorousness, or originality of the research, declaration is not required.
    \item[] Answer: \answerYes{} 
    \item[] Justification: The use of Large Language Models (LLMs) such as DeepSeek-R1, DeepSeek-V3, and Qwen2.5-Coder-7B is a fundamental and core component of this research. The entire paper is about ``reasoning LLMs for RTL generation'' defined as using LLMs to generate RTL code, and proposes an ``training framework'' specifically to train a better LLM in this task. This is detailed throughout the paper, particularly in Sections \ref{sec:introduction}~(Introduction), \ref{sec:method}~(Methods), \ref{sec:experiments}~(Experiments).
    \item[] Guidelines:
    \begin{itemize}
        \item The answer NA means that the core method development in this research does not involve LLMs as any important, original, or non-standard components.
        \item Please refer to our LLM policy (\url{https://neurips.cc/Conferences/2025/LLM}) for what should or should not be described.
    \end{itemize}

\end{enumerate}

\newpage

\appendix

\section{Method Details}

\subsection{Proof and Further Explanation of Theorem 2.1}
\label{sec:proof}

\begin{proof}
Observe that the sequence $Y \to X \to Y'$ forms a Markov chain.  By the Data Processing Inequality (DPI),
\[
I(Y;Y') \;\le\; I(Y;X).
\]
Under the assumption that $E_{Y,Y'}$ holds almost surely, we have $H(Y\mid Y') = 0$, and thus
\[
I(Y;Y') \;=\; H(Y) - H(Y\mid Y') \;=\; H(Y).
\]
It follows that 
\[
H(Y) \;=\; I(Y;Y') \;\leq\; I(Y;X) \;\leq\; H(Y) 
\;\Longrightarrow\;
I(Y;X) \;=\; H(Y) 
\;\Longrightarrow\;
H(Y\mid X) \;=\; 0,
\]
meaning $Y$ is determined by $X$ almost surely and hence $E_{Y,X}$ holds.  

Next, since $H(Y\mid X)=0$ implies $H(X)=H(Y)$ and $I(X;Y')\le H(X)$, a failure of $E_{X,Y'}$ under the NLCDE assumption would force
\[
I(X;Y') \;<\; H(X) \;=\; H(Y),
\]
contradicting $I(Y;Y')=H(Y)$.  Therefore, $E_{X,Y'}$ must also hold almost surely.  Combining these two results gives
\[
E_{Y,X}\;\land\;E_{X,Y'}\quad\text{holds almost surely.}
\]
\end{proof}

\textbf{Remark}: The need for NLCDE in Theorem~2.1 arises because $E_{X,Y} \Rightarrow H(Y\mid X)=0$, but $H(Y\mid X)=0 \not\Rightarrow E_{X,Y}$.  
A counterexample is when $X$ and $Y$ are incorrectly matched with probability one.  
To be more specific (though not fully rigorous, just to aid understanding), if the NL-to-code model wrongly transforms $A$ (e.g., ``design a multiplier'') in the NL domain to $B$ (e.g., ``design an adder'') in the code domain, and transforms $B$ in the NL domain to $A$ in the code domain, while the code-to-NL model maps $A$ in the code domain to $B$ in the NL domain and $B$ in the code domain to $A$ in the NL domain, then $H(Y\mid Y')=0$ can hold without $E_{Y,Y'}$.  
Thus, the NLCDE assumption is necessary to resolve this.

\textbf{Further explanation}: Here we re-emphasize some critical points of this theorem:
\begin{enumerate}
    \item \textbf{Functional Identity}: The theory is built upon the space of code functions ($\mathcal{F}$) and natural language descriptions ($\mathcal{L}$). Different code snippets / NL descriptions that implement / describe the same function (e.g., the same RTL module) are \textbf{considered identical} within $\mathcal{F} ~/~ \mathcal{L}$.
    \item \textbf{Interpretation of Determinism}: The assumption of deterministic mappings ($M_1: \mathcal{F} \rightarrow \mathcal{L}$ and $M_2: \mathcal{L} \rightarrow \mathcal{F}$) models the high functional consistency (not textual uniformity) achieved by capable LLMs. As model capability increases, the mapping from a precise NL description to a core code function becomes increasingly stable (probability converges to $1$).
    \item \textbf{Theoretical Relevance}: Theorem 2.1 establishes that, under these idealized conditions, the functional equivalence (of $Y$ and $Y'$) after the round-trip process ($Y \rightarrow X \rightarrow Y'$) can guarantee the correctness of both problem summarization ($Y \rightarrow X$) and code generation ($X \rightarrow Y'$). This provides a foundational principle explaining why our synthesis loop can bootstrap high-quality, self-consistent data, a principle that is strongly supported by our empirical outcomes.
\end{enumerate}

\subsection{Algorithm Description of Adaptive DAPO}
\label{sec:adaptive_dapo}

In this section, we provide the algorithm description of adaptive DAPO in Algorithm~\ref{alg:adaptive_sampling_final}.
In this algorithm, one epoch means going through the whole training dataset, while one step is to collect enough samples and update the model parameters like standard DAPO~\cite{yu2025dapo}.
Note that we achieve the dynamic batch size by two granularities:
First, we use a step-level ratio $r_{valid}$ to control the generation batch size $b_{gen}$. 
Second, if one generation does not provide enough samples for training, we use another inner-step-level ratio $r_{step}$ to control the generation batch size for the remaining samples.

\begin{algorithm}[H]
\caption{Adaptive DAPO}
\label{alg:adaptive_sampling_final}
\begin{algorithmic}
\REQUIRE Training batch size $b_{train}$, dataset $\mathcal{D}$
\ENSURE Updated $r_{valid}$ and filtered problem pool
\STATE Initialize $r_{valid} \gets 1$

\FOR{epoch $= 1, 2, \ldots$}
    \STATE Shuffle $\mathcal{D}$ \hfill \textit{\small(Epoch reset)}
    \STATE $N_{total} \gets |\mathcal{D}|$, $N_{consumed} \gets 0$
    
    \WHILE{$N_{consumed} < N_{total}$} \hfill \textit{\small(Process epoch)}
        \STATE $\Sigma b_{gen} \gets 0$, $n_{valid} \gets 0$, $r_{step} \gets r_{valid}$
        \WHILE{$n_{valid} < b_{train}$}
            \STATE $b_{remain} \gets b_{train} - n_{valid}$
            \STATE $b_{ge} \gets \lceil b_{remain} / r_{step} \rceil$ \hfill \textit{\small(Dynamic batch)}
            \STATE $\mathcal{D}' \gets \mathcal{D}[N_{consumed} : \min(N_{consumed}+b_{ge}, N_{total})]$
            \STATE Generate $b_{ge}$ samples from $\mathcal{D}'$
            \STATE Update counters: $n_{valid} \gets n_{valid} + v_{new}$, $\Sigma b_{gen} \gets \Sigma b_{gen} + b_{ge}$
            \STATE $r_{step} \gets \min\left(r_{step}, \frac{n_{valid}}{\Sigma b_{gen}} \right)$
        \ENDWHILE
        
        \STATE Update ratio: $r_{valid} \gets \min\left(r_{valid}, \frac{n_{valid}}{\Sigma b_{gen}} \right)$
        \STATE Train DAPO with $b_{train}$ valid samples \hfill \textit{\small(RL step)}
    \ENDWHILE
\ENDFOR
\end{algorithmic}
\end{algorithm}

\section{Parameter Setting}
\label{sec:parameter}

The full parameter setting during the SFT (distillation) stage is shown in Table~\ref{tab:sft_param}, while the full parameter setting during the RL stage is shown in Table~\ref{tab:rl_param}.
During testing, we use a max context length of 16384 and a temperature of 1.0. We set top\_p to 1.0 for VerilogEval and 0.95 for RTLLM.

For RL, the generation batch size in Table~\ref{tab:rl_param} corresponds to train\_batch\_size in verl~\cite{sheng2024hybridflow}, and the training batch size corresponds to ppo\_mini\_batch\_size in verl.
A generation batch size of 128 and training batch size of 64 (with a rollout number of 16) means first generating $128\times16$ samples for 128 problems and updating two times, each with $64\times16$ samples, during one RL step.
Meanwhile, the clip ratio(high), clip ratio(low), overlong penalty factor, and overlong response length in Table~\ref{tab:rl_param} are introduced by DAPO.
Here, the max train response length in Table~\ref{tab:rl_param} corresponds to $L_{max}$ in DAPO, and the overlong response length corresponds to $L_{cache}$.
The overlong penalty in DAPO $P_{length}(y)$ (where y is response length) is defined as:
\begin{align}
    P_{length}(y) = \begin{cases}
        0, &|y| \leq L_{max} - L_{cache} \\
        - \frac{|y| - (L_{max} - L_{cache})}{L_{cache}}, &L_{max} - L_{cache} < |y| \leq L_{max} \\
        -1, &L_{max} < |y|,
    \end{cases}
\end{align}
which is added to the \{0, 1\} reward.

\begin{table*}[h]
\caption{SFT Parameter Setting.}
\label{tab:sft_param}
\centering
\footnotesize
\scalebox{0.9}{
\begin{tabular}{lllll}
\hline
\textbf{Parameter Category} & \textbf{Parameter Name} & \textbf{Value} & \textbf{Parameter Name} & \textbf{Value} \\
\hline
\multirow{1}{*}{Training Mode} & Finetuning Type & Full Parameter & Deepspeed & Zero3 \\
\hline
\multirow{3}{*}{Optimization \& Scheduling} & Epochs & 6 & Learning Rate (LR) & $1 \times 10^{-5}$ \\
& Batch Size & 64 & Optimizer & AdamW \\
& LR Scheduler & Cosine Decay & LR Warmup Ratio & 0.03 \\
& Numerical Precision & BF16 & & \\
\hline
\multirow{1}{*}{Context \& Data Handling} & Max Context Length & 16384 & Packing & True \\
\hline
\end{tabular}
}
\end{table*}

\begin{table*}[h]
\caption{RL Parameter Setting.}
\label{tab:rl_param}
\centering
\footnotesize
\scalebox{0.8}{
\begin{tabular}{lllll}
\hline
\textbf{Parameter Category} & \textbf{Parameter Name} & \textbf{Value} & \textbf{Parameter Name} & \textbf{Value} \\
\hline
\multirow{2}{*}{Batch Size Related} & Generation Batch Size & 128 & Training Batch Size & 64 \\
& Dynamic Batch Size & True & & \\
\hline
\multirow{2}{*}{Rollout Configuration} & Rollout Number & 16 & Rollout Temperature & 1.0 \\
& Rollout Engine & VLLM & Rollout GPU Memory Utilization & 0.8 \\
\hline
\multirow{2}{*}{Optimization \& Regularization} & Learning Rate & $1 \times 10^{-6}$ & Weight Decay & 0.0 \\
& KL Coefficient & 0.0 & KL Loss Coefficient & 0.0 \\
\hline
\multirow{2}{*}{Clipping \& Penalty} & Clip Ratio (High) & 0.28 & Clip Ratio (Low) & 0.2 \\
& Overlong Penalty Factor & 1.0 & & \\
\hline
\multirow{2}{*}{Length Control} & Max Train Response Length (Full) & 16384 & Overlong Response Length & 1024 \\
& Max Generate Response Length & 32768 & & \\
\hline
\multirow{3}{*}{\parbox{3.5cm}{Computation \& \\ Memory Optimization}} & Gradient Clip & 0.5 & Gradient Checkpointing & True \\
& Use Liger Kernel & True & VLLM Enforce Eager & False \\
& Tensor Parallel Size & 4 & & \\
\hline
\multirow{1}{*}{Distributed Training Configuration} & Number of Nodes & 2 & GPUs per Node & 8 \\
\hline
\multirow{1}{*}{Data Processing} & Remove Padding & True & Token Level Loss & True \\
\hline
\multirow{1}{*}{FSDP Related} & FSDP Optimizer Offload & False & FSDP Parameter Offload & False \\
\hline
\end{tabular}
}
\end{table*}

\section{Additional Statistics and Analysis}

\subsection{Benchmark Comparison}

Since there is a notable performance gain difference (especially for the RL phase) of our method between VerilogEval and RTLLM, we provide a deeper analysis of this phenomenon in this section. Given that RTLLM's performance gains stem mainly from reinforcement learning, we focus on distribution differences between our RL dataset and the two benchmarks.

\textbf{Distribution similarity to RTLLM}: We run the instructor-embedding model~\citep{su2022one} for the golden code in our RL dataset, RTLLM (v2), and VerilogEval (v2 spec-to-RTL), then generate a t-SNE distribution plot~\citep{maaten2008visualizing} in Figure~\ref{fig:tsne}. This plot revealed that our RL dataset aligns closely with RTLLM but diverges from VerilogEval for both problems and solutions. 

\begin{figure}[h]
  \centering
  \begin{subfigure}{0.49\linewidth}
    \centering
    \includegraphics[width=\linewidth]{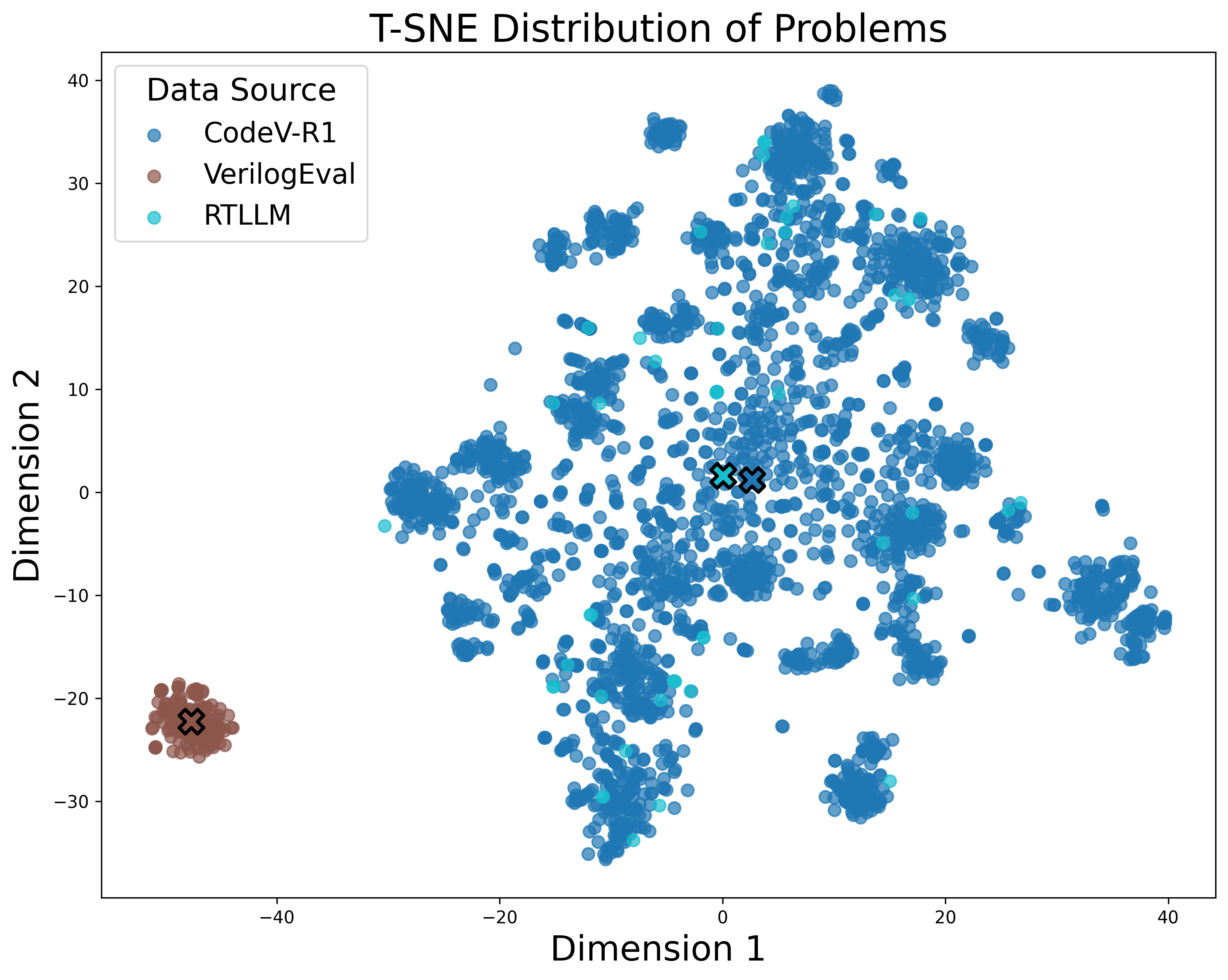}
    \caption{}
  \end{subfigure}
  \begin{subfigure}{0.49\linewidth}
    \centering
    \includegraphics[width=\linewidth]{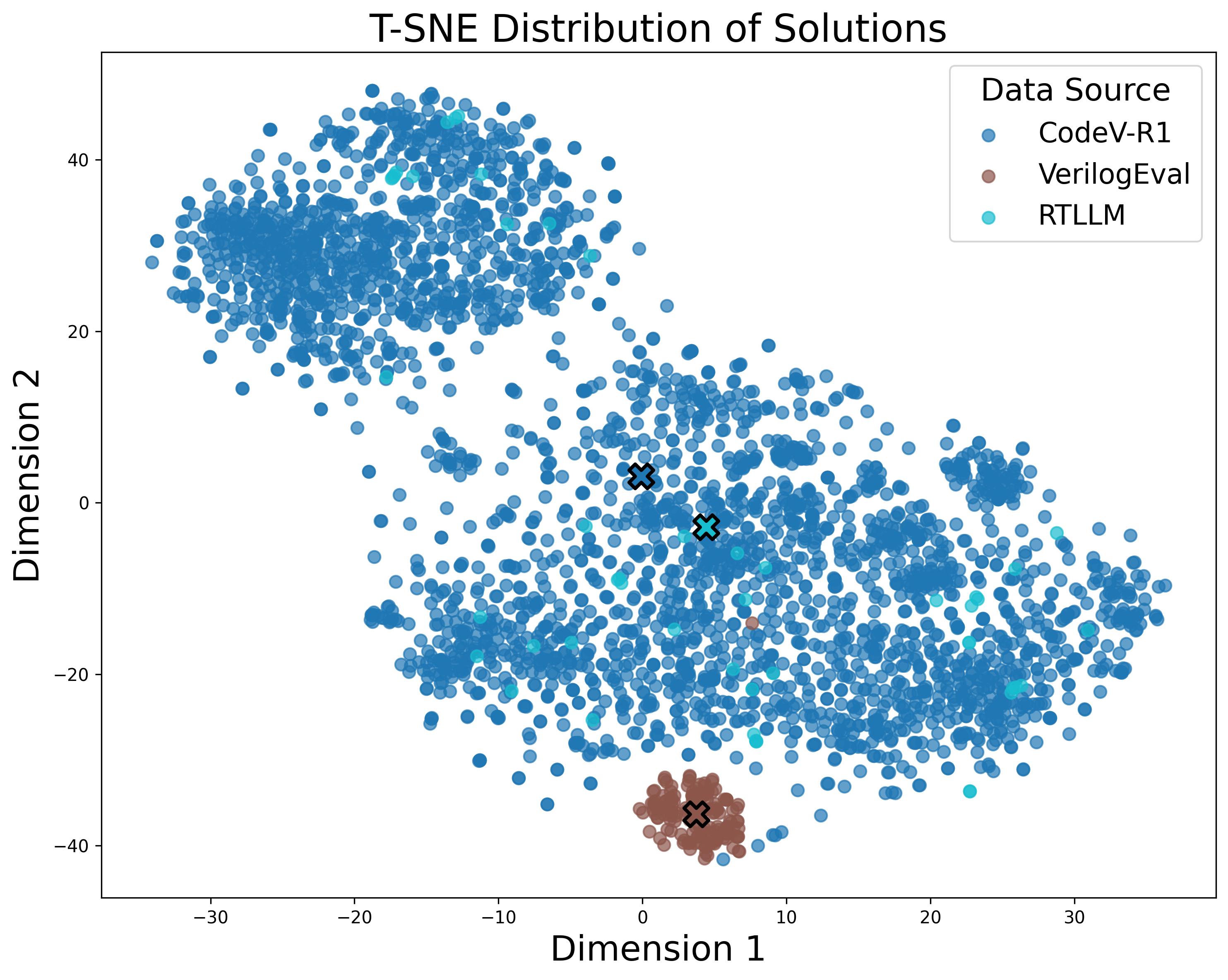}
    \caption{}
  \end{subfigure}
  \caption{\textbf{T-SNE distribution} of \xname~RL dataset, RTLLM (v2), and VerilogEval (v2 spec-to-RTL). \textbf{Left}: Problem (NL) distribution; \textbf{Right}: Solution (code) distribution.}
  \label{fig:tsne}
\end{figure}

\begin{table}[h]
\centering
\caption{Centroid Distance and Similarity between Our RL Dataset and Benchmarks}
\label{tab:tsne_centroid}
\begin{tabular}{lcc}
\hline
& \textbf{RTLLM v2} & \textbf{VerilogEval v2} \\
\hline
Euclidean Distance & 0.1150 & 0.2903 \\
Cosine Similarity & 0.9926 & 0.9547 \\
\hline
\end{tabular}
\end{table}

We also show cosine similarity and Euclidean distance metrics for the embedding centroids between our RL dataset and the benchmarks in Table~\ref{tab:tsne_centroid}.
Our RL dataset’s embedding centroid shows significantly smaller Euclidean distance and closer-to-1 (maximum value) cosine similarity with RTLLM than with VerilogEval. Importantly, this reflects only embedding centroid relationships, not data homogenization or overfitting to RTLLM.

\textbf{Problem type difference}: Additionally, we conducted a detailed case study across both benchmarks and identified that one reason for this performance gain difference is VerilogEval’s heavier use of table/graph-based problems, where our model underperforms significantly. Table~\ref{tab:accuracy_comparison} presents a comparison of our model's accuracy against DeepSeek-R1 on these problem types, including their ratios within benchmarks:

\begin{table}[h]
\centering
\caption{Accuracy Comparison on Table/Graph Problems Across Benchmarks}
\label{tab:accuracy_comparison}
\begin{tabular}{lccc}
\hline
& \textbf{Tables/Graphs (Ours)} & \textbf{Tables/Graphs (DS-R1)} \\
\hline
VerilogEval v2 (Ratio) & 29.49\% & 29.49\% \\
VerilogEval v2 (Accuracy) & 55.43\% & 78.04\% \\
RTLLM v2 (Ratio) & 0\% & 0\% \\
RTLLM v2 (Accuracy) & N/A & N/A \\
\hline
\end{tabular}
\end{table}

This points to a key improvement direction: incorporating more table/graph-specific instruction-response pairs (e.g., KMap, FSM, waveform data as in CraftRTL~\citep{liu2024craftrtl}) into our training dataset.

\textbf{Benchmark complexity comparison}: Additionally, we provide the token count comparison, serving as a proxy for complexity, for these two benchmarks.
Table~\ref{tab:code_length} provides the average number of lines and tokens (comments and blank lines removed) for VerilogEval and RTLLM. From these results, we can clearly see that RTLLM is generally more complex than VerilogEval.

\begin{table}[h]
\centering
\caption{Code Length Comparison between VerilogEval and RTLLM Datasets}
\label{tab:code_length}
\begin{tabular}{lcc}
\hline
Dataset & Average Lines & Average Tokens \\
\hline
VerilogEval v1 Machine & 13.89 & 93.85 \\
VerilogEval v1 Human & 15.82 & 117.35 \\
VerilogEval v2 Code Completion & 16.10 & 121.84 \\
VerilogEval v2 Spec to RTL & 16.12 & 122.31 \\
RTLLM v1.1 & 56.52 & 470.48 \\
RTLLM v2.0 & 46.30 & 403.80 \\
\hline
\end{tabular}
\end{table}

\subsection{Additional Benchmark Statistics}

In this section, we take a close look at the mistake type on VerilogEval v2 and the pass@k metrics of different task types on RTLLM v2.

\begin{table*}[h]
\caption{Comparison of Error Types for VerilogEval v2.}
\label{tab:benchmark_error_type}
\centering
\footnotesize
\scalebox{0.9}{
\begin{tabular}{lcccccccccccccc}
\toprule
\multirow{2}{*}{Model} & \multicolumn{9}{c}{Compiler Errors} & \multicolumn{4}{c}{Runtime Errors} \\
\cmidrule(lr){2-10} \cmidrule(lr){11-14}
 & C & S & w & m & p & e & n & c & Total & R & T & r & Total \\
\midrule
\modelDistill & 69 & 107 & 108 & 27 & 20 & 1 & 0 & 0 & 332 & 1699 & 107 & 19 & 1825 \\
\modelFinal & 63 & 38 & 46 & 22 & 3 & 0 & 1 & 0 & 173 & 1610 & 114 & 1 & 1725 \\
DeepSeek-R1-671B & 37 & 59 & 47 & 0 & 0 & 1 & 0 & 0 & 144 & 1096 & 110 & 5 & 1211 \\
\bottomrule
\end{tabular}
}

\begin{tablenotes}
\item $^{\rm *}$ Error type explanation: 
C -- General Compiler Error; 
S -- Syntax Error; 
w -- Reg Declared as Wire; 
m -- Module Missing; 
p -- Unable to Bind Wire/Reg; 
e -- Explicit Cast Required; 
n -- Sensitivity Problem; 
c -- Unable to Bind Wire/Reg `clk`; 
R -- General Runtime Error; 
T -- Timeout.
r -- Reset Issue; 
\end{tablenotes}

\end{table*}

\textbf{Mistake type:} As shown in Table~\ref{tab:benchmark_error_type}, our RL training notably reduces error rates, particularly for compiler errors. \modelFinal achieves a 48\% reduction in total compiler errors compared to \modelDistill (from 332 to 173), with the most pronounced improvements in syntax errors (S, reduced by 65\% from 107 to 38) and wire declaration issues (w, down 57\% from 108 to 46). Notably, our \modelFinal has a remarkably fewer syntax error (38) compared to DeepSeek-R1 (59) and fewer reset issues (r) (1 vs 5). 
Even so, our \modelFinal still has limitations. For instance, the number of general runtime errors (R) is still notably higher than DeepSeek-R1. This might stem from the RL training data not being suitable for VerilogEval (unlike the great improvement on RTLLM).

\begin{table*}[h]
\caption{Performance Across Different Module Categories on RTLLM v2.}
\label{tab:module_category_pass_k}
\centering
\footnotesize
\scalebox{0.85}{
\begin{tabular}{lcccccccc}
\toprule
\multirow{2}{*}{Model} & \multicolumn{2}{c}{Arithmetic (\%)} & \multicolumn{2}{c}{Control (\%)} & \multicolumn{2}{c}{Memory (\%)} & \multicolumn{2}{c}{Miscellaneous (\%)} \\
\cmidrule(lr){2-3} \cmidrule(lr){4-5} \cmidrule(lr){6-7} \cmidrule(lr){8-9}
 & pass@1 & pass@5 & pass@1 & pass@5 & pass@1 & pass@5 & pass@1 & pass@5 \\
\midrule
\modelDistill & 69.47 & 89.19 & 74.17 & 83.06 & 43.57 & 56.69 & 46.94 & 59.72 \\
\modelFinal & 83.68 & 91.66 & 80.00 & 83.33 & 51.43 & 63.30 & 57.78 & 72.40 \\
DeepSeek-R1-671B & 76.58 & 90.65 & 83.33 & 83.33 & 57.14 & 60.71 & 52.50 & 67.72 \\
\bottomrule
\end{tabular}
}
\end{table*}

\textbf{Accuracy among task types:} Table~\ref{tab:module_category_pass_k} demonstrates the comparative performance across module categories, where \modelFinal shows consistent improvements over \modelDistill while maintaining competitive results against the larger DeepSeek-R1. 
Notably, \modelFinal achieves superior pass@1 rates in all categories over \modelDistill, with particularly strong gains in arithmetic modules (83.68\% vs 69.47\%) and miscellaneous modules (57.78\% vs 46.94\%). It also surpasses DeepSeek-R1 in these two categories. 
Compared with the training dataset classification provided in Figure~\ref{fig:train_data_category}, these two categories occupy a larger portion (arithmetic and others).
This observation suggests that augmenting the training set with high-quality RL data for currently underperforming categories (particularly Memory and Control modules) could be a productive direction for future model improvement.

\subsection{Training Dataset Statistics}

\begin{figure}[h]
  \centering
  \begin{subfigure}{0.49\linewidth}
    \centering
    \includegraphics[width=\linewidth]{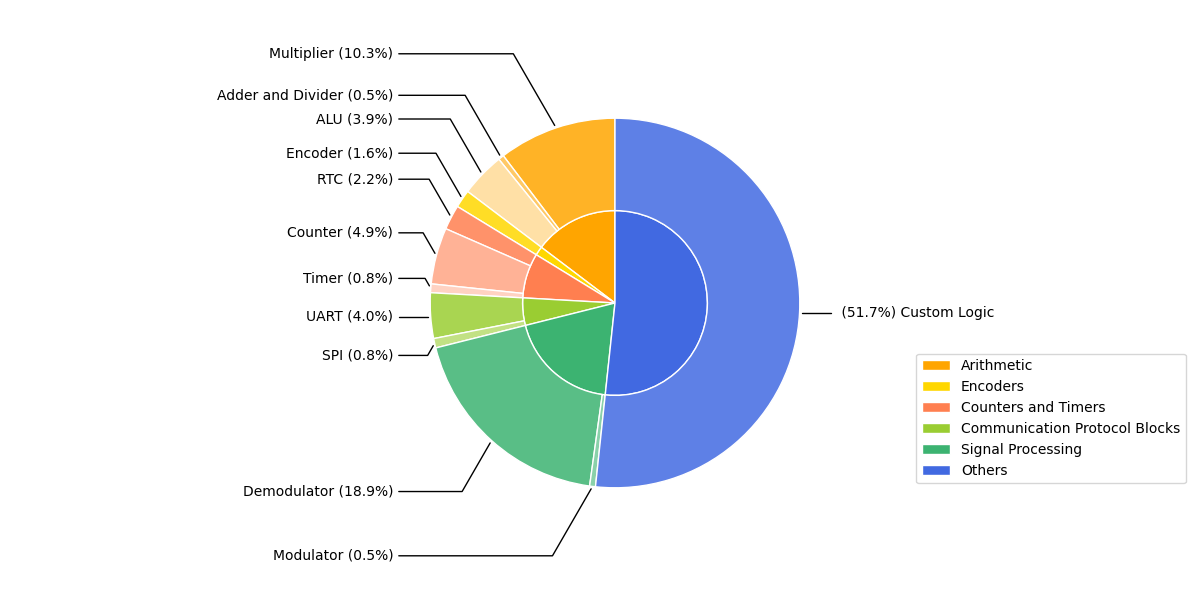}
    \caption{}
  \end{subfigure}
  \begin{subfigure}{0.49\linewidth}
    \centering
    \includegraphics[width=\linewidth]{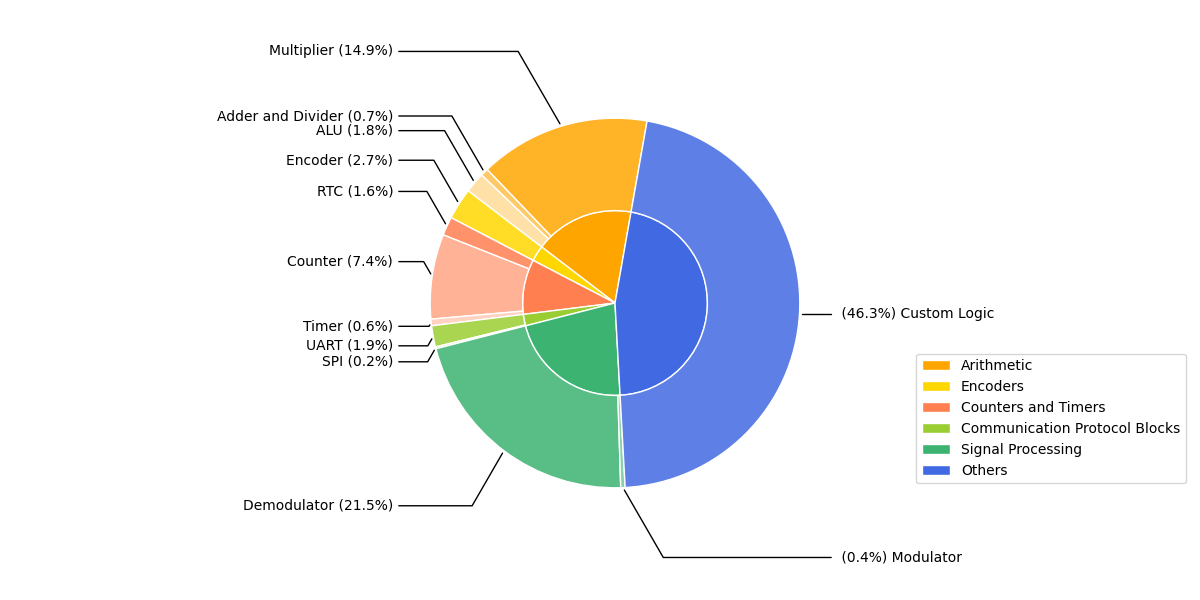}
    \caption{}
  \end{subfigure}
  \caption{\textbf{Problem category distribution.} \textbf{Left}: SFT dataset; \textbf{Right}: RL dataset.}
  \label{fig:train_data_category}
\end{figure}

Figure~\ref{fig:train_data_category} presents the category distribution of our \sizeDistill SFT and \sizeRL RL training datasets (categorized using both questions and answers). While both datasets show comparable distributions, the RL dataset has fewer unclassified problems.

\begin{figure}[h]
    \centering
    \includegraphics[width=\linewidth]{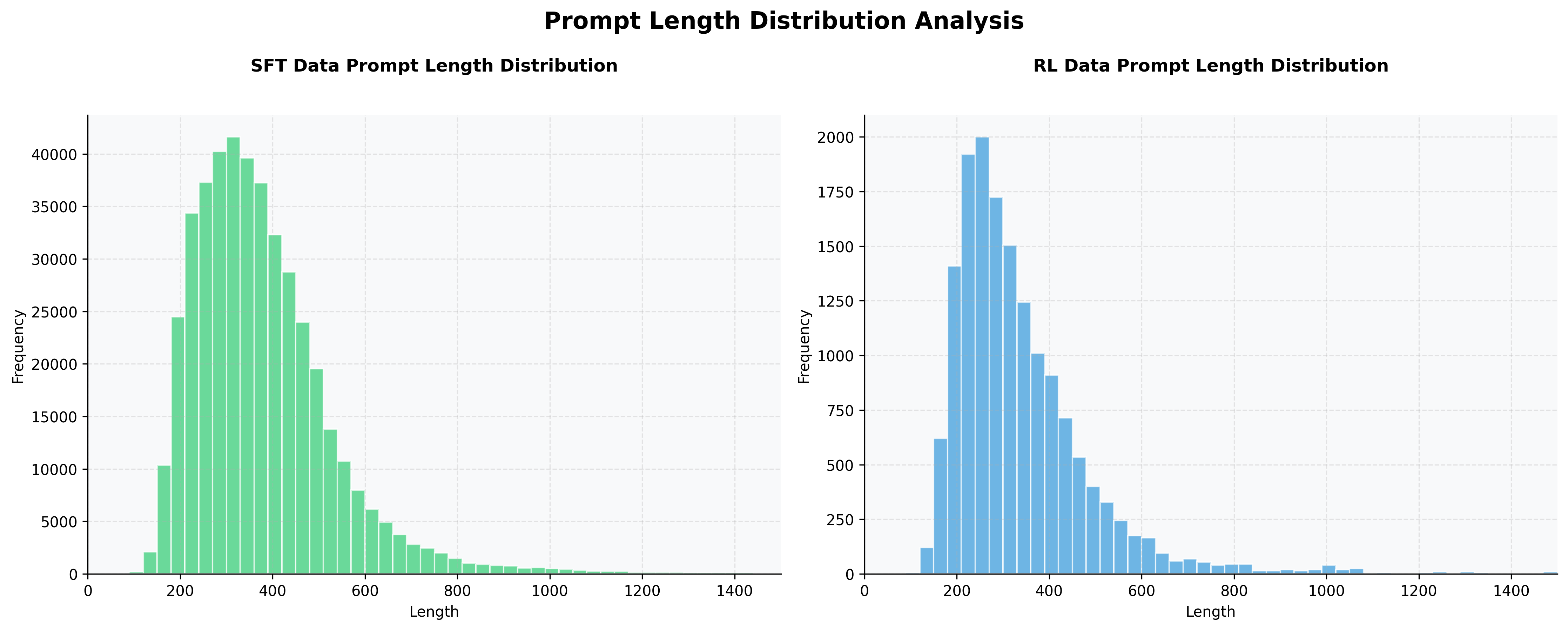}
    \caption{\textbf{Prompt length distribution.} \textbf{Left}: SFT dataset; \textbf{Right}: RL dataset.}
    \label{fig:prompt_dist}
\end{figure}

Figure~\ref{fig:prompt_dist} illustrates the prompt length distribution (in tokens) for our \sizeDistill SFT and \sizeRL RL training datasets, both clipped to a maximum prompt length of 1500 tokens. The figure reveals a sharper distribution for the RL data, indicating shorter and lower-variance prompt lengths compared to the SFT data.
To quantify this observation, we calculated the following statistics:
The average length of SFT data is 377.81 with a standard deviation of 161.30, while the average length of RL data is 336.67 with a standard deviation of 153.88.
These statistics align with the visual trends in the figure.

\begin{figure}[h]
    \centering
    \includegraphics[width=\linewidth]{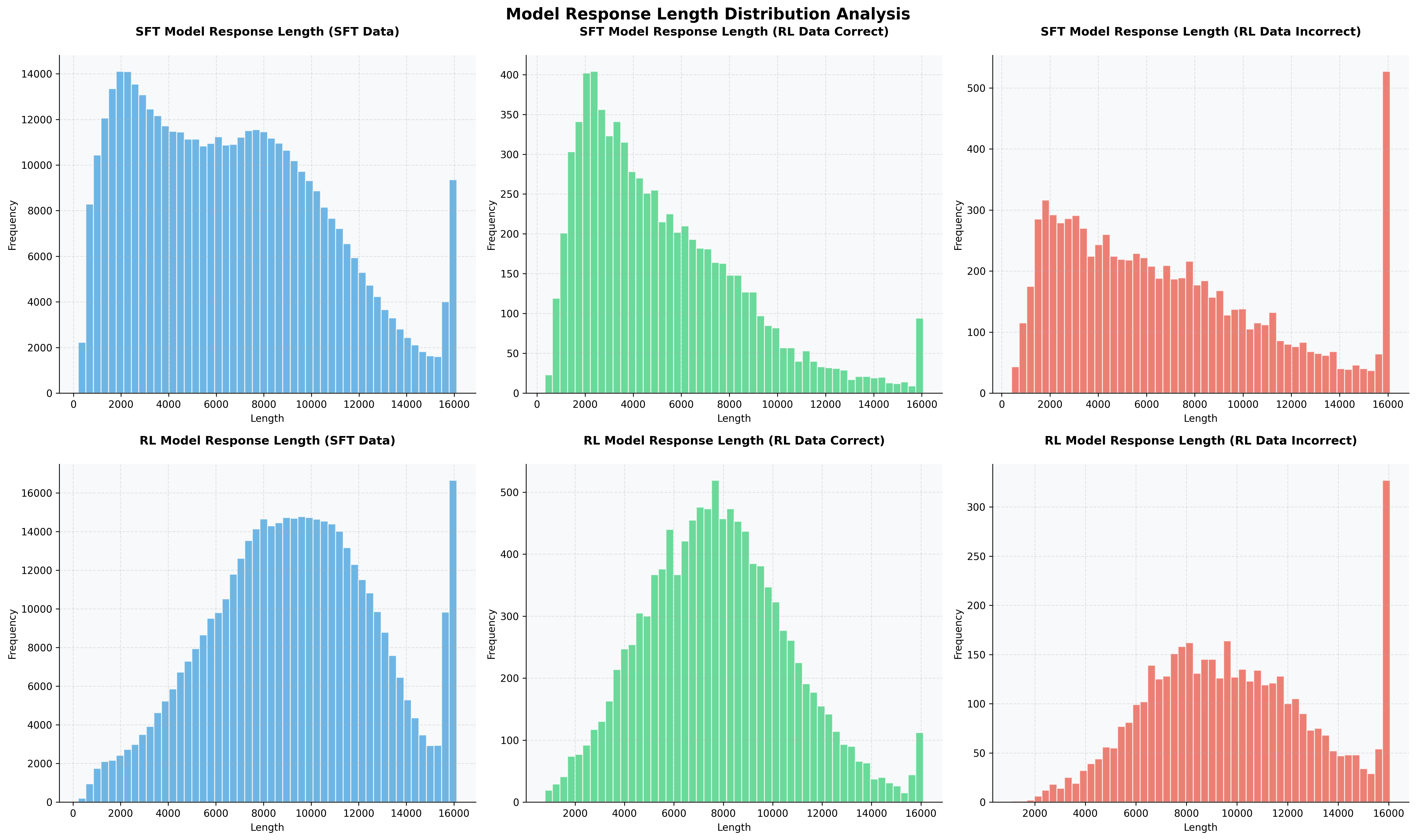}
    \caption{\textbf{Response length distribution.} \textbf{Left}: SFT dataset; \textbf{Middle}: Correct samples in RL dataset; \textbf{Right}: Incorrect samples in RL dataset.}
    \label{fig:response_dist}
\end{figure}

Figure~\ref{fig:response_dist} depicts the response length distributions (in tokens) for \modelDistill and \modelFinal. Note that the maximum context length—the sum of prompt length and response length—is capped at 16384 tokens. Consequently, when responses are truncated, their recorded length is 16384 tokens minus the prompt length, resulting in a somewhat scattered distribution (manifested as the two rightmost bars, instead of one, become longer in the distribution plot).
The response length for \modelFinal exhibits an evident right shift, indicating longer responses after reinforcement learning. Additionally, \modelFinal's response distribution is more symmetric compared to the left-skewed distribution of \modelDistill. The underlying cause of this discrepancy warrants further investigation.
We observe that incorrect samples are significantly longer, with a substantial proportion exceeding the length threshold. Even excluding these overlong samples, incorrect responses remain longer, characterized by a higher peak value (\modelFinal) or a slower post-peak decline (\modelDistill).
An intriguing phenomenon is that \modelFinal has a lower overlong ratio on the RL dataset but a higher ratio on the SFT dataset. This may arise from overfitting the overlong penalty during RL, while \modelFinal's tendency to generate longer responses increases overlong instances on the SFT dataset.

\subsection{Agent Ability Analysis}

This section presents supplementary experiments on agentic integration capabilities for our \modelFinal with the MAGE~\citep{zhao2025mage} framework. We identify two fundamental distinctions between MAGE and our approach: (1) MAGE utilizes golden testbenches for verification while \xname~operates without them, and (2) MAGE's performance exhibits strong dependency on both prompt engineering strategies and the underlying model capabilities.

Hence, we employ a targeted experimental design to isolate and evaluate \xname's efficacy within a multi-agent system. Specifically, we replaced only the RTL generation agent in MAGE with \xname—as this role best aligns with \xname's core capabilities—while maintaining the other agents based on Claude-3.5-Sonnet.

The experimental results in Table~\ref{tab:agent_comparison} demonstrate that \xname~achieves competitive performance (94.87\% Pass@1) compared to the best-performing model (Claude-3.5-Sonnet at 95.51\%), while substantially outperforming other baseline models. This indicates \xname's potential for integration into agent workflows and suggests promising directions for developing end-to-end agent capabilities in future work.

\begin{table}[h]
\centering
\caption{Agent Performance Comparison in MAGE Framework}
\label{tab:agent_comparison}
\begin{tabular}{lc}
\hline
\textbf{Model} & \textbf{Pass@1} \\
\hline
Deepseek-R1-0528 & 80.77\% \\
\DSIII & 88.46\% \\
GPT-4o-240806 & 76.28\% \\
Claude-3.5-Sonnet & 95.51\% \\
\xname (with Claude) & 94.87\% \\
\hline
\end{tabular}
\end{table}

\section{Case Study}
\label{sec:case_study}

\textbf{Comparison Between \modelFinal and \RR}:
In this section, we first present a case study on a specific Verilog problem in VerilogEval-v2 to illustrate the advantage of \modelFinal over \RR in this problem. The problem, reasoning chains, and results are shown in the blocks below.


Both Deepseek-R1 and our model comprehended the problem and engaged in reasoning and reflection, but only our model accurately implemented the logical function defined by the Karnaugh map, while R1's answer executed an incorrect logical function. The key distinction lies in our model's use of the declaration \textit{input [4:1] x}, which directly corresponds to the notation of the Karnaugh map in the problem description. In contrast, R1 employed \textit{input [3:0] x}, leading to confusion in subsequent reasoning. Although our model was initially challenged by this unconventional declaration, it ultimately chose the correct declaration through reflection, avoiding potential index mapping ambiguities. This resulted in significantly clearer code that is less prone to errors.

\begin{tcolorbox}[width=1.0\linewidth, halign=left, colframe=black, colback=white, boxsep=0.01mm, arc=1.5mm, left=2mm, right=2mm, boxrule=1pt, title={Problem Description}]
\footnotesize{
\vspace{5pt}
\textbf{Question:} \\
\vspace{5pt}
Question:
I would like you to implement a module named TopModule with the following
interface. All input and output ports are one bit unless otherwise
specified.

 - input  x (4 bits)
 - output f

The module should implement the function f shown in the Karnaugh map
below. d is don't-care, which means you may choose to output whatever
value is convenient.

~~~~~~~~~~~~x[1]x[2]

  x[3]x[4]   00  01  11  10
  
        ~~00~~~ | d | 0 | d | d |
        
        ~~01~~~ | 0 | d | 1 | 0 |
        
        ~~11~~~ | 1 | 1 | d | d |
        
        ~~10~~~ | 1 | 1 | 0 | d |

Answer:
}
\end{tcolorbox}

\label{subsec:deepseek_r1}
\begin{tcolorbox}[width=1.0\linewidth, halign=left, colframe=black, colback=white, boxsep=0.01mm, arc=1.5mm, left=2mm, right=2mm, boxrule=1pt, fontupper=\fontsize{7pt}{7pt}\selectfont, title={Reasoning Chain (\RR)}]
\footnotesize{
\vspace{5pt}
Okay, so I need to write a Verilog module called TopModule that takes a 4-bit input x and outputs a single bit f based on the given Karnaugh map. Let me start by understanding the problem step by step.

...

Wait, the K-map's rows are x3x4, columns x1x2. Let's see. For example, the first row is x3x4 = 00. So the row is determined by the bits x3 and x4. Similarly, columns are determined by x1 and x2. So each cell in the K-map is determined by the combination of x3, x4, x1, x2. So the four input bits are x1, x2, x3, x4. But how are these mapped to the input x in Verilog? The input x is a 4-bit vector. \begin{color}{red}Let's assume that x is declared as [3:0] x, where x[3] is the first bit (MSB) and x[0] is the last. So perhaps the bits are arranged as x[3] = x1, x[2] = x2, x[1] = x3, x[0] = x4. But this is a guess. Alternatively, the problem might have a typo, and the rows are x[2]x[3], and columns x[0]x[1]. But given the problem statement, I need to proceed with the given labels.\end{color}

...

Now, the goal is to cover all the 1's and 0's with the largest possible groups, using don't-cares as either 1 or 0 to maximize grouping.

First, let's identify the 1's and 0's.

Looking for 1's:

In row 01 (x3=0, x2=1), column 11 (x1=1, x0=1) → 1.

In row 11 (x3=1, x2=1), columns 00 and 01 → 1's.

In row 10 (x3=1, x2=0), columns 00, 01 → 1's.

In row 10, column 11 → 0. Wait, no: row 10, column 11 (x1=1, x0=1) is 0. So that's a 0.

The 0's are:

Row 00, column 01 → 0.

Row 01, columns 00, 10 → 0.

Row 10, column 11 → 0.

The don't-cares can be used to simplify the expression. So for f, we can choose to treat the don't-cares as 1 or 0 to form the largest possible groups.







...

Therefore, the Verilog code for the module would be:

module TopModule (
    input [3:0] x,
    output f
);
\begin{lstlisting}[language=Verilog, basicstyle=\color{red}]
    assign f = (x[3] & ~x[1]) | (~x[3] & x[1] & x[0]);
\end{lstlisting}
endmodule

But let me double-check this.

Another way to check is to test some values.

Test case 1: x3=1, x2=1, x1=0, x0=0 → f should be 1.

In the K-map: row x3x2=11 (row 11), column x1x0=00 → 1. Our expression: (1 \& \~0) | (0 \& 0 \& 0) → 1 | 0 = 1. Correct.

Test case 2: x3=1, x2=0, x1=0, x0=1 → row 10, column 01 → 1. Expression: (1 \& 1) | (0 \& ... ) → 1. Correct.



...

Test case 6: x3=0, x2=0, x1=1, x0=1 → row 00, column 11 (don't-care). Our expression sets this to 1 via \~x3 x1 x0. So f=1 here. Which is acceptable as a don't-care.

Thus, the Verilog code should correctly implement the K-map.
}
\end{tcolorbox}
\begin{tcolorbox}[width=1.0\linewidth, halign=left, colframe=black, colback=white, boxsep=0.01mm, arc=1.5mm, left=2mm, right=2mm, boxrule=1pt, fontupper=\fontsize{7pt}{7pt}\selectfont, title={Results (\RR)}]
\footnotesize{
\vspace{5pt}

...

**Solution**  

\textasciigrave\textasciigrave\textasciigrave verilog
\begin{lstlisting}[language=verilog]
module TopModule (
    input [3:0] x,
    output f
);
\end{lstlisting}
\begin{lstlisting}[language=Verilog, basicstyle=\color{red}]
    assign f = (x[3] & ~x[1]) | (~x[3] & x[1] & x[0]);
\end{lstlisting}
\begin{lstlisting}[language=Verilog]
endmodule
\end{lstlisting}
\textasciigrave\textasciigrave\textasciigrave

...

}
\end{tcolorbox}


\begin{tcolorbox}[width=1.0\linewidth, halign=left, colframe=black, colback=white, boxsep=0.01mm, arc=1.5mm, left=2mm, right=2mm, boxrule=1pt, fontupper=\fontsize{8pt}{8pt}\selectfont, title={Reasoning Chain (\modelFinal)}]
\footnotesize{
\vspace{5pt}
<think>Okay, let's try to tackle this Verilog problem. The goal is to create a module called TopModule with a 4-bit input x and a single-bit output f. The function is determined by the Karnaugh map provided.

...

Let me try to find the logic.

Looking for the minterms where f=1, considering don't cares as 1s.

Wait the f=1 in the K-map.

Looking across the K-map, the entries that have 1 are:

In x3x4=11 row:

Columns x1x2 when:

x1x2=00 (f=1) and x1x2=01 (f=1). 

In x3x4=10 row:

x1x2=00 and 01 (so x3x4=10, x1x2=00 and 01 → which gives x1=0, x2=0 and x1=0, x2=1. So x3x4 can be 10 and x1x2=00 and 01.

Also:

Row x3x4=01 and x1x2=11: f=1.

Row x3x4=01's column x1x2=11 (third column) gives 1.

...

The variables are x[1], x[2], x[3], x[4]. Wait the input is x as four bits. Let's suppose the four bits are x[3:0], but the code might have x[3], x[2], x[1], x[0], but the naming here may vary.

Wait the K-map labels the x as x[1] to x[4]. So in Verilog, if input is declared as [4:1], which is a standard way.

But in Verilog, arrays are usually declared with [width-1 : 0], but the actual variable name is [high:low].

For example, input [4:1] x would have indices x[4], x[3], x[2], x[1], x[0].

But in the Karnaugh map, the rows are x3 and x4, columns are x1 and x2.

But perhaps in the Verilog code, the variables are named x[1],x[2],x[3],x[4]. Let's assume the input is a four-bit vector named x, with each bit labeled x1, x2, x3, x4, such that:

input [4:1] x; // x[1], x[2], x[3], x[4]

But in Verilog, you can declare it as:

input [1:4] x; // but the syntax is different.

Alternatively, input [4:1] x; where x[1] is the first bit, etc.

\begin{color}{blue}
But given the problem statement's variables as x[1] to x[4], the code can be written as:

module TopModule (input [4:1] x, output f);

Then, the variables are:

x[1], x[2], x[3], x[4].
\end{color}

...

So the final code would look like:
module TopModule (input [4:1] x, output f);
\begin{lstlisting}[language=Verilog, basicstyle=\color{blue}]
assign f = (x[3] & ~x[1]) | (~x[3] & x[4] & x[1] & x[2]);
\end{lstlisting}
endmodule
Wait wait, the variables in the second term are x3=0 (because ~x[3], x4=1 (x[4]), x1=1 (x[1]), x2=1 (x[2]).

Yes.

So that's correct.

...

</think>

}
\end{tcolorbox}

\begin{tcolorbox}[width=1.0\linewidth, halign=left, colframe=black, colback=white, boxsep=0.01mm, arc=1.5mm, left=2mm, right=2mm, boxrule=1pt, fontupper=\fontsize{8pt}{8pt}\selectfont, title={Results (\modelFinal)}]
\vspace{5pt}
<answer> \textasciigrave\textasciigrave\textasciigrave verilog
\begin{lstlisting}[language=Verilog]
module TopModule (input [4:1] x, output f);
\end{lstlisting}
\begin{lstlisting}[language=Verilog, basicstyle=\color{blue}]
    assign f = (x[3] & ~x[1]) | (~x[3] & x[4] & x[1] & x[2]);
\end{lstlisting}
\begin{lstlisting}[language=Verilog]
endmodule
\end{lstlisting}
\textasciigrave\textasciigrave\textasciigrave </answer>

\end{tcolorbox}


\textbf{Low-quality Data Example}: Below is an example of low-quality data in our dataset. The original code is just an empty module with some comments, while the summarized problem describes a module with an unsigned 32-bit divider. The problem and the code are inconsistent in this case.

\begin{tcolorbox}[width=1.0\linewidth, halign=left, colframe=black, colback=white, boxsep=0.01mm, arc=1.5mm, left=2mm, right=2mm, boxrule=1pt, fontupper=\fontsize{7pt}{7pt}\selectfont, title={Original Code from GitHub}]
\vspace{5pt}
\begin{lstlisting}[language=Verilog]
module div_unsigned (
        aclr,
        clock,
        denom,
        numer,
        quotient,
        remain);
        input     aclr;
        input     clock;
        input   [31:0]  denom;
        input   [31:0]  numer;
        output  [31:0]  quotient;
        output  [31:0]  remain;
endmodule
// Retrieval info: PRIVATE: new_diagram STRING "1"
// Retrieval info: LIBRARY: lpm lpm.lpm_components.all
// Retrieval info: CONSTANT: LPM_DREPRESENTATION STRING "UNSIGNED"
// Retrieval info: CONSTANT: LPM_HINT STRING "MAXIMIZE_SPEED=6,LPM_REMAINDERPOSITIVE=TRUE"
// Retrieval info: CONSTANT: LPM_NREPRESENTATION STRING "UNSIGNED"
// Retrieval info: CONSTANT: LPM_PIPELINE NUMERIC "11"
// Retrieval info: CONSTANT: LPM_TYPE STRING "LPM_DIVIDE"
// Retrieval info: CONSTANT: LPM_WIDTHD NUMERIC "32"
// Retrieval info: CONSTANT: LPM_WIDTHN NUMERIC "32"
// Retrieval info: USED_PORT: aclr 0 0 0 0 INPUT NODEFVAL "aclr"
// Retrieval info: USED_PORT: clock 0 0 0 0 INPUT NODEFVAL "clock"
// Retrieval info: USED_PORT: denom 0 0 32 0 INPUT NODEFVAL "denom[31..0]"
// Retrieval info: USED_PORT: numer 0 0 32 0 INPUT NODEFVAL "numer[31..0]"
// Retrieval info: USED_PORT: quotient 0 0 32 0 OUTPUT NODEFVAL "quotient[31..0]"
// Retrieval info: USED_PORT: remain 0 0 32 0 OUTPUT NODEFVAL "remain[31..0]"
// Retrieval info: CONNECT: @aclr 0 0 0 0 aclr 0 0 0 0
// Retrieval info: CONNECT: @clock 0 0 0 0 clock 0 0 0 0
// Retrieval info: CONNECT: @denom 0 0 32 0 denom 0 0 32 0
// Retrieval info: CONNECT: @numer 0 0 32 0 numer 0 0 32 0
// Retrieval info: CONNECT: quotient 0 0 32 0 @quotient 0 0 32 0
// Retrieval info: CONNECT: remain 0 0 32 0 @remain 0 0 32 0
// Retrieval info: GEN_FILE: TYPE_NORMAL div_unsigned.v TRUE
// Retrieval info: GEN_FILE: TYPE_NORMAL div_unsigned.inc TRUE
// Retrieval info: GEN_FILE: TYPE_NORMAL div_unsigned.cmp TRUE
// Retrieval info: GEN_FILE: TYPE_NORMAL div_unsigned.bsf TRUE
// Retrieval info: GEN_FILE: TYPE_NORMAL div_unsigned_inst.v TRUE
// Retrieval info: GEN_FILE: TYPE_NORMAL div_unsigned_bb.v TRUE
// Retrieval info: GEN_FILE: TYPE_NORMAL div_unsigned_syn.v TRUE
// Retrieval info: LIB_FILE: lpm
\end{lstlisting}
\end{tcolorbox}

\begin{tcolorbox}[width=1.0\linewidth, halign=left, colframe=black, colback=white, boxsep=0.01mm, arc=1.5mm, left=2mm, right=2mm, boxrule=1pt, fontupper=\fontsize{8pt}{8pt}\selectfont, title={Summarized Problem for The Code}]
\vspace{5pt}
You need to design a Verilog module that performs unsigned division on two 32-bit numbers. The module should have an asynchronous clear signal (aclr) to reset the division operation, and a clock signal (clock) to synchronize the process. The module should take two 32-bit inputs: the divisor (denom) and the dividend (numer). It should produce two 32-bit outputs: the quotient (quotient) and the remainder (remain). The module should be optimized for speed with an 11-stage pipeline and should ensure that the remainder is always positive. This Verilog module, named `div\_unsigned`, has the interface designed as follows:

| Signal Name | Direction | Width | Description                                        |

| ----------- | --------- | ----- | -------------------------------------------------- |

| `aclr`      | Input     | 1     | Asynchronous clear signal to reset the module      |

| `clock`     | Input     | 1     | Clock signal to synchronize the division operation |

| `denom`     | Input     | 32    | 32-bit unsigned divisor input                      |

| `numer`     | Input     | 32    | 32-bit unsigned dividend input                     |

| `quotient`  | Output    | 32    | 32-bit unsigned quotient output                    |

| `remain`    | Output    | 32    | 32-bit unsigned remainder output, always positive  |
\end{tcolorbox}

\section{Prompts}
\label{sec:prompts}

Below shows the prompt for generating instructions given the GitHub codes by \DSIII.

\begin{tcolorbox}[width=1.0\linewidth, halign=left, colframe=black, colback=white, boxsep=0.01mm, arc=1.5mm, left=2mm, right=2mm, boxrule=1pt, fonttitle=\footnotesize, fontupper=\fontsize{7.4pt}{7.4pt}\selectfont, title={Prompt for Generating Questions Given Codes}]
\vspace{3pt}

\#\#\# demonstration1

Please write a description about the code snippet and a problem such that the given code snippet is exactly the answer to this problem.

[The Code Snippet]

\textasciigrave\textasciigrave\textasciigrave verilog
\begin{lstlisting}[language=verilog]
module top_module(
        input ring,
        input vibrate_mode,
        output ringer,
        output motor
);
        assign ringer = ring & ~vibrate_mode;
        assign motor = ring & vibrate_mode;
endmodule
\end{lstlisting}
\textasciigrave\textasciigrave\textasciigrave

[Description]

This Verilog module is a top-level module that controls the ringer and motor of a device. It has four ports: two inputs (ring and vibrate\_mode) and two outputs (ringer and motor).   The ring input is a signal that is used to indicate when the device should be ringing or vibrating. The vibrate\_mode input is a signal that indicates whether the device should be ringing or vibrating.   The ringer output is a signal that is used to control the ringer of the device. It is set to 1 when the device should be ringing and 0 when it should not be ringing. The motor output is a signal that is used to control the motor of the device. It is set to 1 when the device should be vibrating and 0 when it should not be vibrating.   The module uses two assign statements to control the ringer and motor outputs. The first assign statement sets the ringer output to the logical AND of the ring and the logical NOT of the vibrate\_mode inputs. This means that the ringer output will be set to 1 when the ring input is 1 and the vibrate\_mode input is 0, and it will be set to 0 otherwise.   The second assign statement sets the motor output to the logical AND of the ring and the vibrate\_mode inputs. This means that the motor output will be set to 1 when both the ring input and the vibrate\_mode input are 1, and it will be set to 0 otherwise.   This module is used to control the ringer and motor of a device based on the ring and vibrate\_mode inputs. When the ring input is 1 and the vibrate\_mode input is 0, the ringer output will be set to 1 and the motor output will be set to 0. When the ring input is 1 and the vibrate\_mode input is 1, the ringer output will be set to 0 and the motor output will be set to 1.  

[Problem]

Suppose you are designing a circuit to control a cellphone's ringer and vibration motor. Whenever the phone needs to ring from an incoming call (input ring), your circuit must either turn on the ringer (output ringer = 1) or the motor (output motor = 1), but not both. If the phone is in vibrate mode (input vibrate\_mode = 1), turn on the motor. Otherwise, turn on the ringer. This Verilog module, named `top\_module`, has the interface designed as follows:

| Signal Name  | Direction | Width | Description                                                             |

|--------------|-----------|-------|-------------------------------------------------------------------------|

| `ring`       | Input     | 1     | Indicates when the device should be ringing or vibrating                |

| `vibrate\_mode` | Input   | 1     | Selects vibrate mode; when high, the device vibrates instead of ringing  |

| `ringer`     | Output    | 1     | Controls the device's ringer; high to ring, low otherwise               |

| `motor`      | Output    | 1     | Controls the device's motor; high to vibrate, low otherwise             |

~\\

\#\#\# demonstration2

Please write a description about the code snippet and a problem such that the given code snippet is exactly the answer to this problem.

[The Code Snippet]

\textasciigrave\textasciigrave\textasciigrave verilog
\begin{lstlisting}[language=verilog]
module top_module(
        input clk,
        input resetn,        
        input [1:0] byteena, 
        input [15:0] d,      
        output reg [15:0] q);
    always @(posedge clk) begin
            if (!resetn)
                    q <= 0;
            else begin
                    if (byteena[0])
                            q[7:0] <= d[7:0];  
                    if (byteena[1])
                            q[15:8] <= d[15:8];
            end
    end
endmodule
\end{lstlisting}
\textasciigrave\textasciigrave\textasciigrave

[Description]

This Verilog module is a top-level module which is used to write data to a 16-bit register. It has five inputs and one output. The inputs are a clock signal (clk), a reset signal (resetn), a 2-bit byte enable signal (byteena), and a 16-bit data signal (d). The output is a 16-bit register (q).  The module is triggered by a rising edge of the clock signal. When the reset signal is low, the register is reset to 0. Otherwise, the data signal is written to the register based on the byte enable signal. If the first bit of the byte enable signal is high, the lower 8 bits of the data signal are written to the lower 8 bits of the register. If the second bit of the byte enable signal is high, the upper 8 bits of the data signal are written to the upper 8 bits of the register.

[Problem]
Create 16 D flip-flops. It's sometimes useful to only modify parts of a group of flip-flops. The byte-enable inputs control whether each byte of the 16 registers should be written to on that cycle. byteena[1] controls the upper byte d[15:8], while byteena[0] controls the lower byte d[7:0]. resetn is a synchronous, active-low reset. All DFFs should be triggered by the positive edge of clk. This Verilog module, named `top\_module`, has the interface designed as follows:

| Signal Name | Direction | Width | Description                        |

|-------------|-----------|-------|------------------------------------|

| `clk`         | Input     | 1     | Clock signal                       |

| `resetn`      | Input     | 1     | Active low reset signal            |

| `byteena`     | Input     | 2     | Byte enable signal                  |

| `d`           | Input     | 16    | Data input signal                   |

| `q`           | Output    | 16    | 16-bit register output              |
\end{tcolorbox}

\begin{tcolorbox}[width=1.0\linewidth, halign=left, colframe=black, colback=white, boxsep=0.01mm, arc=1.5mm, left=2mm, right=2mm, boxrule=1pt, fonttitle=\footnotesize, fontupper=\fontsize{7.5pt}{7.5pt}\selectfont]
\vspace{3pt}

(continued)

\#\#\# demonstration3

Please write a description about the code snippet and a problem such that the given code snippet is exactly the answer to this problem.

[The Code Snippet]

\textasciigrave\textasciigrave\textasciigrave verilog
\begin{lstlisting}[language=verilog]
module top_module(
        input clk,
        input reset,
        output reg [3:0] q);

        always @(posedge clk)
                if (reset || q == 10)
                        q <= 1;
                else
                        q <= q+1;

endmodule
\end{lstlisting}
\textasciigrave\textasciigrave\textasciigrave

[Description]

This top Verilog module is a simple counter that increments its output q by one every clock cycle. It has 3 inputs, a clock (clk), a reset signal (reset), and an output register (q). The output register is a 4-bit register, meaning it can store values from 0 to 15.   The module is triggered on the rising edge of the clock signal. When the reset signal is active, the output register is set to 1. If the reset signal is not active, the output register is incremented by one. When the output register reaches 10 (1010 in binary), it is reset to 1. This process is repeated every clock cycle.

[Problem]

Make a decade counter that counts 1 through 10, inclusive. The reset input is active high synchronous, and should reset the counter to 1. This Verilog module, named `top\_module`, has the interface designed as follows:

| Signal Name | Direction | Width | Description                                                   |

|-------------|-----------|-------|---------------------------------------------------------------|

| `clk`       | Input     | 1     | Clock signal that triggers the counter on its rising edge    |

| `reset`     | Input     | 1     | Active-high synchronous reset signal to initialize the counter|

| `q`         | Output    | 4     | 4-bit register output representing the current count (1-10)   |

~\\

\#\#\# demonstration4
Please write a description about the code snippet and a problem such that the given code snippet is exactly the answer to this problem.

[The Code Snippet]

\textasciigrave\textasciigrave\textasciigrave verilog
\begin{lstlisting}[language=verilog]
module top_module(
        input clk,
        input reset,
        output reg [4:0] q);

        logic [4:0] q_next;
        always @(q) begin
                q_next = q[4:1];
                q_next[4] = q[0];
                q_next[2] ^= q[0];
        end

        always @(posedge clk) begin
                if (reset)
                        q <= 5'h1;
                else
                        q <= q_next;
        end
endmodule
\end{lstlisting}
\textasciigrave\textasciigrave\textasciigrave

[Description]

The top module has 3 inputs and 1 output, where the inputs are clk, reset, and output is q.
The module has 2 always blocks to define the state transition of q and the logic description.
The state transitions are defined in the first always block, which is triggered when q changes.
In the first always block, q\_next is assigned with q[4:1], which is the value of q except the LSB bit.
Then, q\_next[4] is assigned with q[0], which is the LSB bit.
Lastly, q\_next[2] is xored with q[0].
The second always block is triggered at positive edge of clk.
If reset is active, q is assigned with 5'h1, which is the reset value.
If reset is inactive, q is assigned with q\_next, which is the state transition.
The port connections of instantiated modules are shown above.
The module takes clk, reset, and q as input.
q is a 5 bit output, which is assigned with 5'h1 at reset and q\_next at positive edge of clk.

[Problem]

A linear feedback shift register is a shift register usually with a few XOR gates to produce the next state of the shift register. A Galois LFSR is one particular arrangement where bit positions with a "tap" are XORed with the output bit to produce its next value, while bit positions without a tap shift. If the taps positions are carefully chosen, the LFSR can be made to be "maximum-length". A maximum-length LFSR of n bits cycles through 2**n-1 states before repeating (the all-zero state is never reached). Build a 5-bit maximal-length Galois LFSR with taps at bit positions 5 and 3. The active-high synchronous reset should reset the LFSR output to 1. This Verilog module, named `top\_module`, has the interface designed as follows:

| Signal Name | Direction | Width | Description                                                      |

|-------------|-----------|-------|------------------------------------------------------------------|

| `clk`       | Input     | 1     | Clock signal that triggers state transitions on rising edges.   |

| `reset`     | Input     | 1     | Active-high synchronous reset signal to initialize the LFSR.      |

| `q`         | Output    | 5     | Current state of the LFSR, representing a 5-bit value (`00001` to `11111`). |

\end{tcolorbox}

\begin{tcolorbox}[width=1.0\linewidth, halign=left, colframe=black, colback=white, boxsep=0.01mm, arc=1.5mm, left=2mm, right=2mm, boxrule=1pt, fonttitle=\footnotesize, fontupper=\fontsize{7.5pt}{7.5pt}\selectfont]
\vspace{3pt}

(continued)

\#\#\# demonstration5

Please write a description about the code snippet and a problem such that the given code snippet is exactly the answer to this problem.

[The Code Snippet]

\textasciigrave\textasciigrave\textasciigrave verilog
\begin{lstlisting}[language=verilog]
module top_module (
        input [99:0] in,
        output [98:0] out_both,
        output [99:1] out_any,
        output [99:0] out_different
);

        assign out_both = in & in[99:1];
        assign out_any = in | in[99:1];
        assign out_different = in^{in[0], in[99:1]};
endmodule
\end{lstlisting}
\textasciigrave\textasciigrave\textasciigrave

[Description]

This Verilog module is used to compare two input signals and generate three output signals. The first input signal is a 100-bit wide vector, and the second input signal is the same vector shifted by one bit. The module has three output signals, out\_both, out\_any, and out\_different.   The out\_both signal is generated by performing a bit-wise AND operation between the two input signals. This will result in a 99-bit wide vector, where each bit is 1 only if both the corresponding bits of the two input signals are 1.  The out\_any signal is generated by performing a bit-wise OR operation between the two input signals. This will result in a 100-bit wide vector, where each bit is 1 if either of the corresponding bits of the two input signals is 1.  The out\_different signal is generated by performing a bit-wise XOR operation between the two input signals. This will result in a 100-bit wide vector, where each bit is 1 only if the corresponding bits of the two input signals are different. The first bit of the out\_different signal is generated by performing a bit-wise XOR operation between the first bit of the first input signal and the last bit of the second input signal.

[Problem]

You are given a 100-bit input vector in[99:0]. We want to know some relationships between each bit and its neighbour: 

// (1) out\_both: Each bit of this output vector should indicate whether both the corresponding input bit and its neighbour to the left are '1'. For example, out\_both[98] should indicate if in[98] and in[99] are both 1. Since in[99] has no neighbour to the left, the answer is obvious so we don't need to know out\_both[99]. 

// (2) out\_any: Each bit of this output vector should indicate whether any of the corresponding input bit and its neighbour to the right are '1'. For example, out\_any[2] should indicate if either in[2] or in[1] are 1. Since in[0] has no neighbour to the right, the answer is obvious so we don't need to know out\_any[0].   

// (3) out\_different: Each bit of this output vector should indicate whether the corresponding input bit is different from its neighbour to the left. For example, out\_different[98] should indicate if in[98] is different from in[99]. For this part, treat the vector as wrapping around, so in[99]'s neighbour to the left is in[0].

This Verilog module, named `top\_module`, has the interface designed as follows:

| Signal Name     | Direction | Width | Description                                                                                      |

|-----------------|-----------|-------|--------------------------------------------------------------------------------------------------|

| `in`            | Input     | 100   | 100-bit input vector for analyzing bit relationships                                             |

| `out\_both`      | Output    | 99    | Each bit indicates if both the corresponding input bit and its left neighbor are '1'             |

| `out\_any`       | Output    | 99    | Each bit indicates if either the corresponding input bit or its right neighbor is '1'            |

| `out\_different` | Output    | 100   | Each bit indicates if the corresponding input bit is different from its left neighbor, circularly |

~\\

\#\#\# Instruction

Please write a description about the code snippet and a problem such that the given code snippet is exactly the answer to this problem.

[The Code Snippet]

\textasciigrave\textasciigrave\textasciigrave verilog

\textcolor{red}{\{The Given Code Snippet\}}

\textasciigrave\textasciigrave\textasciigrave

\#\#\# Response

\end{tcolorbox}

Our prompt begins by presenting five distinct demonstrations. Each demonstration first provides a description of a code snippet, followed by the generation of a corresponding problem. We then prompt the model (\DSIII) to generate a problem similarly based on the given code snippet colored in red. This process mirrors the multi-level summarization mechanism in CodeV~\cite{codev}.

We also show the system prompt we use during training (both SFT and RL) and testing (on benchmarks) as below.

\begin{tcolorbox}[width=1.0\linewidth, halign=left, colframe=black, colback=white, boxsep=0.01mm, arc=1.5mm, left=2mm, right=2mm, boxrule=1pt, fonttitle=\footnotesize, 
title={System Prompt for Training and Testing}]
\vspace{3pt}

You are a helpful assistant. The assistant first thinks about the reasoning process in the mind and then provides the user with the answer. The reasoning process and answer are enclosed within <think> </think> and<answer> </answer> tags, respectively, i.e., <think> reasoning process here </think><answer> answer here </answer>.  Now the user asks you to write verilog code. After thinking, when you finally reach a conclusion, enclose the final verilog code in \textasciigrave\textasciigrave\textasciigrave verilog \textasciigrave\textasciigrave\textasciigrave within <answer> </answer> tags. i.e., 
<answer>
\textasciigrave\textasciigrave\textasciigrave verilog\textbackslash n
module top\_module(in, out, ...); ... 
\textasciigrave\textasciigrave\textasciigrave
</answer>.

\end{tcolorbox}








\section{Broader Impacts}
\label{sec:braoder_impacts}

Through distillation from \RR and reinforcement learning, \modelFinal even outperforms \RR-671B on RTLLM v1.1 and RTLLM v2, while outperforming previous Verilog-domain state-of-the-art models (typically 7~15B) by 12$\sim$21 \% on RTLLM v1.1 and v2.
Through these results, our work demonstrates the promising potential of reinforcement learning for improving circuit design.

However, analogous to other code generation models, \modelFinal may produce code that misaligns with user intentions or even be misused for unintended purposes. As comprehensively analyzed in broader impact studies~\cite{chen2021evaluating}, such risks include but are not limited to: 

\begin{enumerate}
    \item Functional misalignment: Generated code might superficially satisfy requirements but fail to execute as intended, particularly in safety-critical circuit designs.
    \item Security vulnerabilities: The model could inadvertently generate insecure code (e.g., flawed logic or backdoors), which poses risks in hardware deployment.
    \item Misuse in malicious contexts: Lower barriers to code generation may facilitate the creation of obfuscated or harmful designs, especially as model capabilities scale.
\end{enumerate}

Given the potentially severe consequences of such issues in hardware systems, we strongly recommend that users:
\begin{enumerate}
    \item Conduct rigorous functional verification and security audits for all generated code.
    \item Implement access controls and usage monitoring to mitigate abuse risks.
    \item Adopt a principle of "human-in-the-loop" oversight, particularly for high-stakes applications.
\end{enumerate}

\section{Limitations and Future Work}
\label{sec:limitation_and_future_work}

This work has several limitations, and we primarily discuss two key aspects that also define our future direction: (1) The automated testbench generation framework can only improve the semantic consistency between code and NL in the probabilistic sense. The synthetic dataset generated by our method both for SFT and RL may still contain a small amount of low-quality data, which could potentially impact the model's performance. (2) Collecting data with reasoning processes for SFT requires a general reasoning model (e.g., \RR), which inherently depends on the teacher model’s reasoning capabilities. This dependency poses greater challenges in specialized domains where the teacher model’s performance is suboptimal, as its limitations in such contexts may directly impact the quality of the collected data. Besides, this process might be financially costly.

Additionally, from an application perspective, it is promising to focus on exploring the potential of reasoning LLMs to tackle more complex hardware development tasks beyond RTL code generation in the future, such as PPA performance optimization and analog circuit synthesis.

\end{document}